\documentclass[twocolumn,natbib]{svjour3}  

\usepackage{amsmath,amsfonts,bm,bbm}
\usepackage{times, amssymb, verbatim, comment}
\usepackage{color,pifont}
\usepackage{hyperref,url,booktabs}
\usepackage{algorithm,algorithmic,multirow,hhline}
\usepackage[hang,flushmargin]{footmisc}
\usepackage[pdftex]{graphicx}
\usepackage{graphicx,graphics}
\usepackage[font=small]{caption}
\usepackage{subcaption}
\usepackage{mathrsfs,url}
\input{mysymbol.sty}

\renewcommand{\theenumi}{\roman{enumi}}%

\newcommand{\xmark}{\text{\ding{55}}}



\newtheorem{assumption}{\hspace{0pt}\bf Assumption}

\newcommand{\INDSTATE}[1][1]{\STATE\hspace{3mm}}
\newcommand{\INDSTATED}[1][1]{\STATE\hspace{6mm}}

\begin{document}

\title{Consistent Online Gaussian Process Regression \\ Without the Sample Complexity Bottleneck}

\author{Alec Koppel, Hrusikesha Pradhan, Ketan Rajawat}       

\institute{Alec Koppel \at
              CISD, U.S. Army Research Laboratory, Adelphi, MD 20783, USA  \\
              \email{alec.e.koppel.civ@mail.mil}           
           \and
           Hrusikesha Pradhan \at
              Dept. of Electrical Engineering, Indian Institute of Technology Kanpur
Kanpur UP 208016 INDIA\\
\email{hpradhan@iitk.ac.in}
    \and
    Ketan Rajawat \at
    Dept. of Electrical Engineering, Indian Institute of Technology Kanpur
Kanpur UP 208016 INDIA\\
\email{ketan@iitk.ac.in} \\
{\scriptsize{Part of this work (no proofs and preliminary experiments) appeared as \citep{koppel2019consistent}. In contrast, here we provide full justification of all theoretical results and extensive experiments.} }} 

\date{Received: 23 Apr 2020 / Accepted: date}

\maketitle

\begin{abstract}%
Gaussian processes provide a framework for nonlinear nonparametric Bayesian inference widely applicable across science and engineering. Unfortunately, their computational burden scales cubically with the training sample size, which in the case that samples arrive in perpetuity, approaches infinity. This issue necessitates approximations for use with streaming data, which to date mostly lack convergence guarantees. Thus, we develop the first online Gaussian process approximation that preserves convergence to the population posterior, i.e., \emph{asymptotic posterior consistency}, while ameliorating its intractable complexity growth with the sample size. We propose an online compression scheme that, following each a posteriori update, fixes an error neighborhood with respect to the Hellinger metric centered at the current posterior, and greedily tosses out past kernel dictionary elements until its boundary is hit.
 We call the resulting method \emph{Parsimonious Online Gaussian Processes} (POG).  
For diminishing error radius, {asymptotic statistical stationarity is achieved} (Theorem \ref{theorem:pog_consistency}\ref{theorem:diminishing}) at the cost of unbounded memory in the limit. On the other hand, for constant error radius, POG converges to a neighborhood of {stationarity} (Theorem \ref{theorem:pog_consistency}\ref{theorem:constant}) but with finite memory at-worst determined by the metric entropy of the feature space (Theorem \ref{theorem:pog_parsimony}). {Here stationarity refers to the distributional distance between sequential marginal posteriors approaching null with the time index.} Experimental results are presented on several nonlinear regression problems which illuminates the merits of this approach as compared with alternatives that fix the subspace dimension defining the history of past points. 
%
\end{abstract}

\section{Introduction}\label{sec:mdps}
Gaussian process  regression \citep{rasmussen2004gaussian}, or kriging \citep{krige1951statistical}, is a framework for nonlinear nonparametric Bayesian inference widely used in chemical processing \citep{kocijan2016modelling}, robotics \citep{deisenroth2015gaussian}, and machine learning \citep{rasmussen2004gaussian}, among other applications. Unfortunately, in the batch setting, its complexity scales cubically $N^3$ with the training sample size $N$. Moreover, in the online/stochastic setting where the total number of training samples may be infinite $N\rightarrow \infty $ or samples continually arrive in a streaming fashion, this complexity tends to infinity. Thus, popular perception is that GPs cannot work online. 
In this work, we upend this perception by developing a method that approximately preserves the distributional properties of GPs in the online setting while yielding a worst-case {per-step and per-evaluation} complexity that is sample size independent, and is instead determined by the metric entropy of the feature space \citep{zhou2002covering}.

\begin{table*}[]
\centering
\renewcommand\tabcolsep{1.8pt}
\caption{Summary of related work on approximate Gaussian processes. {We provide the first convergence result for a sparse GP in the {online} setting.}}\label{Tab1}
\begin{tabular}{cccccccccccc}
\toprule
\multirow{1}{*}{Likelihood/Objective} & \multicolumn{1}{c}{Compression Statistic} & \multicolumn{1}{c}{Subspace Size ($M$)}  & \multicolumn{1}{c}{Convergence} & \multicolumn{1}{c}{Reference}  \\
\midrule
Posterior          &    None                 &   Infinite        & $\checkmark$     & \citep{ghosal2000convergence,van2008rates}          \\[5pt]

Prior    &  Log-likelihood  	& Fixed         & $\xmark$        &  \citep{smola2001sparse,keerthi2006matching}  \\[5pt]
Prior    &  RKHS metric  	& Fixed         & $\xmark$        &  \citep{csato2002sparse}  \\[5pt]
Prior    &  Information Gain      & Fixed   &  $\xmark$         & \citep{seeger2003fast}    \\[5pt]
Prior    &  KL Divergence/Entropy  & Dynamic     & $\xmark$        & \citep{mcintire2016sparse}     \\[5pt]
Variational {Lower}-Bound                                     & KL  Divergence                  & Fixed   & $\xmark$  &          \citep{titsias2009variational,bui2017streaming}      \\[5pt]
Variational {Lower}-Bound                                     & KL Divergence                   & Fixed   & $\checkmark$  &          \citep{mcallester1999pac,burt2019rates}    \\[5pt]
Posterior          &    Hellinger Metric                 &   Dynamic        & $\checkmark$      &This Work          \\[5pt]

\bottomrule
\end{tabular}
\end{table*}

Consider the batch setting $N<\infty$. The complexity bottleneck comes from the posterior mean and covariance's dependence on a data matrix, or kernel dictionary, that accumulates all past observations. To address this memory explosion, one must choose a subset of $M\ll N$ possible model points from the training set and ``project the posterior distribution" onto the ``subspace" defined by these samples. Various criteria for projection have been proposed based on information gain \citep{seeger2003fast}, greedy compression \citep{smola2001sparse}, Nystr{\"o}m sampling \citep{williams2001using}, probabilistic criteria \citep{solin2014hilbert,mcintire2016sparse,bauer2016understanding}, and many others \citep{bui2017streaming}. Unfortunately, {to the best of our knowledge}, when $N\rightarrow\infty$ {and samples are observed sequentially, i.e., for the online setting, all of these methods lack any (even approximate) notion of optimality, which for posterior inference is } known as \emph{asymptotic posterior consistency} \citep{barron1999consistency,ghosal2000convergence}. Posterior consistency means that the empirical distribution tends to its population counterpart, and hence the Gaussian mean and covariance tend to their Bayesian ground truth, as the sample size tends to infinity. 

Now, let's shift focus to GPs in the online setting \citep{roberts2013gaussian}. When written in a sequential manner, the posterior mean and covariance updates are parametrized by a kernel dictionary that grows by one every time a new sample arrives. Clearly, this is untenable for situations where samples arrive ad infinitum. The question, then, is how to selectively retain a subset of past training examples to ensure the posterior distribution nearly locks onto the population probability measure, while bringing the memory under control. 

In the case of supervised learning (frequentist regression/classification), one may break the curse of kernelization \citep{koppel2019parsimonious,koppel2018projected} by seeking a Lyapunov function of the sequential estimates, and projecting the function iterates in such a way to throw away as much information as possible while preserving the Lyapunov function's per-step evolution. In supervised learning, this may be achieved by ensuring the projection satisfies a stochastic descent property \citep{Bottou1998,nemirovski2009robust}. By contrast, in Bayesian a posteriori estimation, the Lyapunov function is the distance between the current empirical distribution and its population counterpart,  which contracts to null as the number of samples becomes large \citep{ghosal2000convergence}.

More specifically, posterior contraction means the empirical posterior probability measure based upon $t$ samples and its population counterpart become closer according to some metric. The Hellinger metric computed in closed form for multivariate Gaussians \citep{abou2012note}, and many rate results can be specified with respect to this metric \citep{choi2007posterior,van2008rates,van2009adaptive,kruijer2010adaptive,stuart2018posterior}. Thus, we define as a Lyapunov function of the sequential posterior estimates the Hellinger distance between the empirical posterior probability from $t$ samples and its population counterpart. Through this insight, our contributions are to:

\begin{itemize}
\item develop an algorithm called Parsimonious Online Gaussian Processes (POG), which executes online compression by fixing an error neighborhood of radius $\epsilon_t$ in terms of the Hellinger distance around the current posterior, and greedily tossing out as many model points as possible, while staying inside this neighborhood (Section \ref{sec:online}). This compression uses a custom destructive variant of matching pursuit \citep{Pati1993,Vincent2002} that quantifies error with respect to the Hellinger metric. Here $\epsilon_t$ denotes the compression budget \citep{koppel2019parsimonious}, and the greedy compression may be viewed as a hard-thresholding projection \citep{parikh2014proximal} onto subspaces of posterior distributions. 
\item establish that POG {converges to a statistical stationary point under a non-expansive condition on the single-step likelihood change}\footnote{{We define statistical stationarity as the distributional distance between subsequent marginal posterior distributions approaching null (Sec \ref{sec:convergence}}).} with respect to the population posterior when the compression budget approaches null $\eps_t\rightarrow 0$ as $t\rightarrow \infty$ (Theorem \ref{theorem:pog_consistency}\ref{theorem:diminishing}, Section \ref{sec:convergence})
\item establish that under constant compression budget $\eps_t=\eps>0$, the Bayesian posterior converges to within an $\epsilon$-neighborhood of {a statistical stationary point under a non-expansive condition on the single-step likelihood \\ change }(Theorem \ref{theorem:pog_consistency}\ref{theorem:constant}). Moreover, with this budget selection, the size of the kernel dictionary is at-worst finite, and defined by the metric entropy of the feature space (Theorem \ref{theorem:pog_parsimony}). Thus, we obtain {rigorous} tradeoff of \emph{approximate consistency} and \emph{model parsimony}.
\item experimentally (Section \ref{sec:experiments}) on the \textit{boston}, \textit{abalone} and \textit{kin-40k} data sets we observe favorable tradeoffs between performance and complexity relative to fixed subspace approaches in the offline \citep{snelson2006sparse} and online setting \citep{csato2002sparse}. 
%
\end{itemize}

{\bf \noindent Context} We expand upon related approaches, as seeking tractable approximations of online GPs has been studied for years. One approach is to focus on the exact GP likelihood, and compute the inference exactly, in the spirit of John Tukey's truism: ``Far better an approximate answer to the right question...than an exact answer to the wrong question..." \citep{tukey1962future}. This approach,  Fully Independent Training Conditional (FITC), originating in \citep{csato2002sparse}, derives an additive decomposition of the GP likelihood, and then projects the posterior parameters onto a subspace of fixed size. This approach, called Sparse Online Gaussian Processes (SOGP), has given rise to numerous variations that modify the projection criterion using, e.g., information gain \citep{seeger2003fast}, the negative log-likelihood of the posterior \citep{smola2001sparse,keerthi2006matching}, or regularizations of the likelihood \citep{mcintire2016sparse}. The commonality of these approaches is that their posterior subspace dimension is fixed in advance. In the offline setting, one may additionally optimize over the retained points in the subspace, an approach dubbed pseudo-input search \citep{snelson2006sparse}.

An alternative approach is to employ a variational {lower}-bound of the GP likelihood, e.g., the Variational Free Energy (VFE) \citep{titsias2009variational} or Expectation Propagation \citep{csato2002gaussian}, which as illuminated by \citep{bui2017streaming}, can be employed in the online setting by subsuming the GP approximation into streaming variational Bayes \citep{broderick2013streaming}. Online variational approaches remain an active area of research, which are well-summarized in \citep{liu2018gaussian}. Recently, performance of various approximations has been characterized in terms of approximate convergence to a variational {lower}-bound on the population likelihood \citep{mcallester1999pac,burt2019rates}. Moreover, \citep{toth2019variational} employs them for time-series (non-i.i.d.) analysis, and builds geometric intuition for the approximation subspace \citep{shi2019sparse}. 

Predominately, the aforementioned approaches fix the {maximal} subspace size, which is not inherently a drawback, except {that fixing a priori how many points are permitted may preclude} convergence in terms of any law of large numbers \citep{ghosal2000convergence}, or a {lower}-bound thereof \citep{reeb2018learning}. See \citep{roberts2013gaussian,bauer2016understanding} for reviews of challenges and approaches for the online setting, or \citep{quinonero2007approximation, banerjee2012efficient,rasmussen2004gaussian}[Ch. 8] for experimentation on offline approximations. {More specifically, these approaches permit the active set of points that form the prior for the next posterior inference to grow or shrink according to their statistical significance, but only up to an a priori fixed upper-limit. Doing so may or may not include enough points required for convergence.}

By contrast, in this work, we do not fix the memory that defines the number of past retained points, but instead fix the compression-induced error, and allow the subspace dimension to grow/shrink with the importance of new information. Most similar to this work is \citep{mcintire2016sparse}, which experimentally (but not theoretically) substantiates the merits of allowing the subspace dimension to continually evolve. In this work, we derive a compression rule directly motivated by laws of large numbers for GPs \citep{ghosal2000convergence,van2008rates}, whose tradeoffs we establish both in theory and practice. In particular, the proposed compression routine provably {nearly} converges to {approximate consistency}, and retains finitely many points in its approximate posterior. See Table \ref{Tab1} for a summary.
\section{Gaussian Process Regression}\label{sec:prob}

In GP regression \citep{krige1951statistical,rasmussen2004gaussian} there is some function $f(\bbx)$ that models the relationship between random variables $\bbx\in\ccalX\subset\reals^p$ and $y\in\ccalY\subset\reals$, i.e., $\hat{y}=f(\bbx)$, that we are trying to estimate upon the basis of $N$ training examples $\ccalS=\{\bbx_n, y_n\}_{n=1}^N$. Unlike in expected risk minimization (ERM), we do not learn this estimator by solving an optimization problem that defines its merit of fitness, but instead assume that this function $f(\bbx)$ follows some particular parametrized family of distributions, and then we seek to estimate those parameters. 

In particular, for GPs, we place a {fixed} \emph{prior} on the distribution of $\bbf_{\ccalS}=[f(\bbx_n),\cdots,f(\bbx_N)]$ as a Gaussian distribution, namely, $\bbf_{\ccalS}\sim \ccalN( \bb0,\bbK_N)$. Here $\ccalN(\bbmu,\bbSigma)$ denotes the multivariate Gaussian distribution in $N$ dimensions with mean vector $\bbmu\in\reals^N$ and covariance $\bbSigma\in\reals^{N\times N}$.  In GP regression, the covariance $\bbK_N=[\kappa(\bbx_m,\bbx_n)]_{m,n=1}^{N,N}$ is constructed from a distance-like kernel function $\kappa:\ccalX\times\ccalX\rightarrow\reals$ defined over the product set of the feature space. The kernel expresses some prior about how to measure distance between points, a common example of which is itself the Gaussian, $[\bbK_N]_{mn}=\kappa(\bbx_m,\bbx_n)=\exp\{-\|\bbx_m-\bbx_n\|^2/c^2\}$ with bandwidth hyper-parameter $c$. 
 
 In standard GP regression, we further place a zero-mean Gaussian prior on the noise that corrupts $\bbf_\ccalS$ to form observations $\bby=[y_1,\cdots,y_N]$, i.e., $\mathbb{P}(\bby\given \bbf_\ccalS) =\ccalN(\bbf_{\ccalS},\sigma^2\bbI)$ where $\sigma^2$ is some variance parameter. In Section \ref{sec:convergence} we use $\Pi$ to denote the measure associated with the Gaussian prior. We may integrate out the prior on $\bbf_{\ccalS}$ to obtain the marginal likelihood for $\bby$ as 
 \begin{equation}\label{eq:GP_prior}
\mathbb{P}( \bby\given \ccalS) = \ccalN(\bb0,\bbK_N + \sigma^2 \bbI)
 \end{equation}
 Upon receiving a new data point $\bbx_{N+1}$, we make a Bayesian inference $\hat{y}_{N+1}$, not by defining a point estimate $f(\bbx_{N+1})=\hat{y}_{N+1}$. Instead, we formulate the entire \emph{posterior} distribution for $y_{N+1}$ as:
 \begin{align}\label{eq:GP_posterior}
\mathbb{P}( y_{N+1} \given \ccalS \cup \bbx_{N+1})= \ccalN\Big(\bbmu_{N+1\given \ccalS},\bbSigma_{N+1\given \ccalS}\Big) 
 \end{align}
 where the mean and covariance in \eqref{eq:GP_posterior} are given by 
 \begin{align}\label{eq:GP_posterior_params}
\bbmu_{N+1\given \ccalS}&=\bbk_{\ccalS}(\bbx_{N+1})[\bbK_N+ \sigma^2 \bbI]^{-1}\bby_N\nonumber  \\
\bbSigma_{N+1\given \ccalS}&=\kappa(\bbx_{N+1},\bbx_{N+1})   \\
&- \bbk_{\ccalS}^T(\!\bbx_{N+1}\!)\![\bbK_N\!\!+\! \sigma^2 \bbI]^{-1}\!\bbk_{\ccalS}(\!\bbx_{N+1}\!) + \sigma^2  \nonumber
 \end{align}
 The expressions \eqref{eq:GP_posterior_params} are standard \citep{rasmussen2004gaussian}[Ch. 2]. Here $\bbk_{\ccalS}(\bbx) = [\kappa(\bbx_1, \bbx) ; \cdots \kappa(\bbx_N, \bbx)]$ denotes the empirical kernel map.
While this approach to sequential Bayesian inference provides a powerful framework for fitting a mean and covariance envelope around observed data, it requires for each $N$ the computation of $\bbmu_{N+1\given \ccalS}$ and $\bbSigma_{N+1\given \ccalS}$ which crucially depend on computing the inverse of the kernel matrix $\bbK_N$ every time a new data point arrives. It is well-known that matrix inversion has cubic complexity $\ccalO(N^3)$ in the variable dimension $N$, which may be reduced through use of {partial} Cholesky factorization \citep{foster2009stable} or subspace projections \citep{banerjee2012efficient} combined with various compression criteria such as information gain \citep{seeger2003fast}, mean square error \citep{smola2001sparse}, integral approximation for Nystr{\"o}m sampling \citep{williams2001using}, probabilistic criteria \citep{mcintire2016sparse,bauer2016understanding}, and many others \citep{bui2017streaming}. {Doing so requires one to fix an upper-bound on the number of of points $M \ll N < \infty$ that defines the posterior, which may or may not be sufficient for convergence -- rates of convergence for sparse GPs have been discerned only for the \emph{offline setting} \citep{mcallester1999pac,burt2019rates}.}

 However, even with this complexity reduction, if one tries to run GP regression in true streaming applications where the sample size is not necessarily finite $N\rightarrow \infty$, any computational savings is eventually rendered moot useless unless one ensures the {per update} complexity remains independent of the sample size. Thus, our objective is to find an approximate GP whose memory is sample complexity independent, yet is as close as possible in distribution to the fully infinite (as $N\rightarrow \infty $) dense GP. Next, we shift focus to developing such a method.

 \section{Online Gaussian Processes}\label{sec:online}
 
 In this section, we derive our new algorithm Parsimonious Online Gaussian Process (POG) which is nothing more than a rewriting of the posterior update \eqref{eq:GP_posterior} in time-series manner, and constructing an online sparsification rule that ensures its complexity remains under control, while also preserving {a relaxed notion of posterior consistency}. Define the time series of observations as $\mathcal{S}_t=\{\bbx_u,y_u\}_{u\leq t}$, and then rewrite \eqref{eq:GP_posterior} in a time-varying way in terms of $\mathcal{S}_t\cup\{\bbx_{t+1}\}$ as 
  \begin{align}\label{eq:GP_posterior_t}
 \bbmu_{t+1\given \ccalS_t}&=\bbk_{\ccalS_t}(\bbx_{t+1})[\bbK_t+ \sigma^2 \bbI]^{-1}\bby_t\nonumber  \\
 \bbSigma_{t+1\given \ccalS_t}&=\kappa(\bbx_{t+1},\bbx_{t+1})   \\
& -  \bbk_{\ccalS_t}^T(\bbx_{t+1}\!)[\bbK_t+ \sigma^2 \bbI]^{-1}\bbk_{\ccalS_t}(\bbx_{t+1}\!)+ \sigma^2.\nonumber
 \end{align}

Observe that this update causes the dictionary $\bbX_t:=[\bbx_1;\cdots;\bbx_{t}]\in \mathbb{R}^{p\times t}$ to grow by one at each iteration i.e. $\bbX_{t+1}=[\bbX_t;\bbx_{t+1}]\in\reals^{p\times (t+1)}$, and that the posterior estimates at time $t+1$ use \emph{all past observations} $\{\bbx_u\}_{u\leq t}$ in the kernel dictionary $\bbX_{t+1}$. Subsequently, we refer to the number of columns in the dictionary matrix as the \emph{model order} $M_t$. The GPs posterior estimates at time $t$ have model order $M_t=t$. For future reference, denote the posterior distributions of $y_t$ and $y_{t+1}$ defined by \eqref{eq:GP_posterior_t}, as $\rho_t=\mathbb{P}( y_{t} \given \ccalS_{t-1} \cup \bbx_{t}) $ and $\rho_{t+1}=\mathbb{P}( y_{t+1} \given \ccalS_{t} \cup \bbx_{t+1}) $, respectively. {We further note $\rho_t(\cdot)= \mathbb{P}( y_{t} \given \ccalS_{t-1} \cup \cdot ) $ as the the full posterior viewed as a function to be evaluated at a given test point $\tbx\in\ccalX$.

In this work, we focus on the online (or \emph{one-shot}) setting, where in-situ performance is prioritized over extrapolation at points far away from  $\bbx_{t+1}$. Therefore, we focus on ensuring the quality of the sequence $\rho_t$, the GP evaluated at $\bbx_{t+1}$, but not the more general function space view in terms of $\rho_t(\cdot)$. It may be possible to generalize the techniques here from sequential marginal posteriors to the entire posterior sequence $\rho_t(\cdot)$, doing so would likely require a subsampling step or moment matching scheme (similar to particle filters or the resampling step in importance sampling \citep{smith2013sequential}), which must be conducted in a fundamentally \emph{offline} manner. This time-scale separation precludes the development of such algorithms for online settings, but is nonetheless a meritorous direction for future work.}

\subsection{Compressing the Posterior Distributions}

Our memory-reduction technique relies on the following conceptual link between sparsity in stochastic programming and Bayesian inference: a typical Lyapunov (potential) function of a online supervised learning algorithm is the mean sub-optimality, which quantifies the ``energy" in an optimization algorithm. If the method has been appropriately designed, the Lyapunov function is negative semi-definite, and thus flows to the minimum energy state, yielding convergence to optimality \citep{khalil1996noninear}. When sparsity considerations are additionally present, the role of a proximal \citep{atchade2014stochastic}/hard-thresholding \citep{nguyen2017linear} operator in tandem with the descent direction must be analyzed to establish decrement in expectation.

Motivated by the {stabililty analysis} of stochastic gradient hard-thresholding algorithms and {related notions} in sparse nonparametric function estimation \citep{koppel2019parsimonious}, we seek a Lyapunov function of the sequential Bayesian {a posteriori} update for GPs, and develop a custom hard-thresholding projection that preserves its per-step behavior. In Bayesian inference, there is no optimization iterate or gradient, but instead only posterior estimates and their associated distributions. Thus, to quantify convergence, we propose measuring how far the empirical posterior is from {the true posterior distribution} according to some {distributional} metric, which actually quantifies distance from \emph{posterior consistency} \citep{barron1999consistency}. {More specifically, $y_t$ is assumed to follow some unknown generating process $f_0$ corrupted by noise, whose associated posterior distribution is denoted as $\rho_0$. These entities are then used to define a notion of distributional convergence in Lemma \ref{lemma:posterior_consistency}.}

\begin{algorithm}[t]
	\caption{Parsimonious Online GPs (POG)}
	\begin{algorithmic}
		\label{alg:pog}
		\REQUIRE $\{\bbx_t,y_t,\epsilon_t \}_{t=0,1,2,...}$
		\STATE \textbf{initialize} {observation} noise $\sigma^2$,  empty dictionary $\bbD_0 = []$
		\FOR{$t=0,1,2,\ldots$}
		\STATE Obtain independent training pair $(\bbx_{t+1}, y_{t+1})$
		\STATE Update posterior mean and covariance estimates \eqref{eq:GP_posterior_D}
		\begin{align*}
		\bbmu_{t+1\given \bbD_{t} }&=\bbk_{\bbD_{t}}(\bbx_{t+1})[\bbK_{\bbD_{t},\bbD_{t}}+ \sigma^2 \bbI]^{-1}\bby_t\nonumber  \\
		\bbSigma_{t+1\given \bbD_{t}}&\!\!=\!\kappa(\bbx_{t+1},\bbx_{t+1}) 
		 \nonumber \\
		& - \bbk^T_{\bbD_{t}}\!(\!\bbx_{t+1}\!) [\bbK_{\bbD_{t},\bbD_{t}}\!\!+\! \sigma^2 \bbI]^{-1}\bbk_{\bbD_{t}}(\bbx_{t+1})+ \sigma^2
		\end{align*}\vspace{-4mm}
		\STATE Revise dictionary via {a posteriori} update $\tbD_{t+1}\! = \![\bbD_{t},\bbx_{t+1}]$
		\STATE Compress w.r.t. Hellinger metric via Algorithm \ref{alg:komp} 
	\begin{align*}(\bbmu_{\bbD_{t+1}},\bbSigma_{\bbD_{t+1}}&,\bbD_{t+1}) \nonumber \\
		&=\textbf{DHMP}(\bbmu_{t+1\given \bbD_{t}},\bbSigma_{t+1\given \bbD_{t}},\tbD_{t+1},\eps_t)
		\end{align*}
		
		\ENDFOR
	\end{algorithmic}
\end{algorithm}
We fix our choice of metric as the Hellinger distance due to the fact that it is easily computable for two multivariate Gaussians \citep{abou2012note}. For two continuous distributions $\nu(\bbx)$ and $\lambda(\bbx)$ over feature space $\ccalX$, the Hellinger distance is defined as
\begin{equation}\label{eq:hellinger}
d_H(\nu,\lambda) := \frac{1}{2}\int_{\bbx \in \ccalX} \left(\sqrt{\nu(\bbx)} - \sqrt{\lambda(\bbx)}\right)^2 d\bbx \; , 
\end{equation}
which, when both distributions are normal $\nu = \ccalN(\bbmu_1,\Sigma_1)$ and  $\lambda = \ccalN(\bbmu_2,\Sigma_2)$ \citep{abou2012note}, takes the form
\begin{align}\label{eq:hellinger_gaussian}
&d_H(\nu,\lambda)  \\
&\!\!=\! \sqrt{1 \!-\! \frac{|\Sigma_1 |^{1/4} |\Sigma_2|^{1/4}}{| \bar{\Sigma}|} \!\exp\!\left\{\!\! -\frac{1}{8}\!\left(\mu_1 \!-\! \mu_2\right)\bar{\Sigma}^{-1} \left(\mu_1 - \mu_2 \right)  \!\right\}\! } \nonumber 
\end{align}
where $\bar{\Sigma}=(\Sigma_1 + \Sigma_2)/2$.  The distribution defined by the Bayesian updates \eqref{eq:GP_posterior} converges to the underlying population distribution in Hellinger distance with probability 1 with respect to the population posterior distribution (see \citep{choi2007posterior}[Theorem 6], \citep{van2008rates}[Theorem 3.3], \citep{kruijer2010adaptive}[Theorem 2]).
These posterior contraction results motivate the subsequence design of the compression sub-routine. If we select some other kernel dictionary $\bbD\in \mathbb{R}^{p\times M}$ rather than $\bbX_t$ for some model order $M$, the only difference is that the kernel matrix $\bbK_t$ in \eqref{eq:GP_posterior_t} and the empirical kernel map $\bbk_{\ccalS}(\cdot)$ are substituted by $\bbK_{\bbD\bbD}$ and $\bbk_{\bbD}(\cdot)$, respectively, where the entries of $[\bbK_{\bbD,\bbD}]_{mn}=\kappa(\bbd_m,\bbd_n)$, $\bbk_{\bbD}=[\kappa(\bbd_1, \cdot); \cdots; \kappa(\bbd_M, \cdot)]$ and $\{\bbd_m\}_{m=1}^M\subset\{\bbx_u\}_{u\leq t}$. Let's rewrite \eqref{eq:GP_posterior_t} for sample $(\bbx_{t+1},y_{t+1})$ with $\bbD$ as the kernel dictionary rather than $\bbX_{t}$ as
\begin{align}\label{eq:GP_posterior_D}
\bbmu_{t+1\given \bbD}&=\bbk_{\bbD}(\bbx_{t+1})[\bbK_{\bbD,\bbD}+ \sigma^2 \bbI]^{-1}\bby_t\nonumber  \\
\bbSigma_{t+1\given \bbD}&=\kappa(\bbx_{t+1},\bbx_{t+1})  \\
&- \! \bbk_{\bbD}^T(\!\bbx_{t+1}\!)[\bbK_{\bbD,\bbD}\!+\! \sigma^2 \bbI]^{-1}\bbk_{\bbD} (\bbx_{t+1})+ \sigma^2.\nonumber
\end{align}
The question, then, is how to select a sequence of dictionaries $\bbD_{t}\in\mathbb{R}^{p\times M_t}$ whose $M_t$ columns
comprise a subset of those of $\bbX_t$ in such a way to preserve asymptotic posterior consistency.

\begin{algorithm}[t]
	\caption{Destructive Hellinger Matching Pursuit (DHMP) \hspace{-2mm}}
	\begin{algorithmic}
		\label{alg:komp}
		\REQUIRE  posterior dist. $\rho_{\tbD}$ defined by dict. ${\tbD} \in \reals^{p \times \tilde{M}}$, mean $\bbmu_{t+1\given {\tbD} }$ covariance $\bbSigma_{t+1\given {\tbD}}$, budget  $\epsilon_t > 0$ \\
		\STATE \textbf{initialize} mean $\bbmu_{\tbD}=\tbmu_{t+1\given \tbD }$, cov. $\bbSigma_{\bbD}=\tbSigma_{t+1\given \tbD }$,\\ dict. $\bbD = {\tbD}$ w/ indices $\ccalI$ with  model order $M=\tilde{M}$
		\WHILE{candidate dictionary is non-empty $\ccalI \neq \emptyset$}
		{\FOR {$j=1,\dots,\tilde{M}$}
			\STATE Compute mean $\bbmu_{\bbD_{-j}}$, cov. $\bbSigma_{\bbD_{-j}}$ w/o dict. point $\bbd_j$ 
			\STATE Find error with dict. element $\bbd_j$ removed \eqref{eq:hellinger_gaussian} \vspace{-2mm}
			$$\gamma_j = d_H(\rho_{\bbD_{-j}}, \rho_{\tbD} ) \; .$$
			%
			\ENDFOR}
		\STATE Find minimal error dict. point: $j^* = \argmin_{j \in \ccalI} \gamma_j$
		\INDSTATE{{\bf{if }} minimal approximation error exceeds $\gamma_{j^*}> \epsilon_t$}
		\INDSTATED{\bf stop} 
		\INDSTATE{\bf else} 
		
		\INDSTATED Prune dictionary $\bbD\leftarrow\bbD_{\ccalI \setminus \{j^*\}}$
		\INDSTATED Revise set $\ccalI \leftarrow \ccalI \setminus \{j^*\}$, model order ${M} \leftarrow {M}-1$.
		%
		\INDSTATE {\bf end}
		\ENDWHILE	
		\RETURN dictionary $\bbD$ such that $d_H(\rho_{\bbD} , \rho_{\tbD})\leq \eps_t$
	\end{algorithmic}
\end{algorithm}


Suppose we have a dictionary $\bbD_t\in\mathbb{R}^{p\times M_t}$ at time $t$ and observe point $\bbx_{t+1}$. We compute its associated posterior distribution $\rho_{\bbD_{t}}:=\ccalN\big(\bbmu_{t+1\given \bbD_{t}},\bbSigma_{t+1\given \bbD_{t}}\big)$, where the expressions for the mean and covariance can be obtained by substituting $\bbD=\bbD_{t}$ into \eqref{eq:GP_posterior_t}, assuming that $\bbD_t$ has already been chosen. {We define $\rho_{\bbD_{t}}(\cdot)$ analogously to $\rho_t(\cdot)$ following equation \eqref{eq:GP_posterior_t}.}
 
 We propose compressing dictionary $\tbD_{t+1}$ of model order $M_t +1$ to obtain a dictionary $\bbD_{t+1}$ of smaller model complexity $M_{t+1} \leq M_t +1 $ by executing the update in \eqref{eq:GP_posterior_D}, fixing an error neighborhood centered at  $\ccalN\big(\bbmu_{t+1\given \bbD_{t}},\bbSigma_{t+1\given \bbD_{t}}\big)$ in the Hellinger
 metric. Then, we prune dictionary elements greedily with respect to the Hellinger metric until we hit the boundary of this error neighborhood. This is a destructive variant of matching pursuit \citep{Pati1993,Vincent2002} that has been customized to operate with the Hellinger distance, and is motivated by the fact that we
 can tune its stopping criterion to assure that the intrinsic distributional properties of the Bayesian
 update are almost unchanged. We call this routine  Destructive Hellinger Matching Pursuit (DHMP) with budget parameter $\epsilon_t$. 
 
 DHMP with compression budget $\epsilon_t$, summarized in Algorithm \ref{alg:komp}, operates by taking as input a kernel dictionary $\tbD_{t+1}$ and associated posterior mean and covariances estimates $\bbmu_{t+1\given \bbD_{t}},\bbSigma_{t+1\given \bbD_{t}}$ and initializing its approximation as the input. Then, it sequentially and greedily removes kernel dictionary elements $\bbd_j$ according to their ability to that cause the least error in Hellinger distance. Then, it terminates when the resulting distribution $\rho_{\bbD}$ hits the boundary of an $\eps_t$ error neighborhood in Hellinger
distance from its input. This procedure with input $\tbD_{t+1}\in \reals^{p\times M_{t}+1}$, $\bbmu_{t+1\given \bbD_t},\\ \bbSigma_{t+1\given \bbD_t}$ and parameter $\eps_t$, is summarized in Algorithm \ref{alg:komp}. 

The full algorithm, Parsimonious Online Gaussian Processes (POG), is summarized as Algorithm \ref{alg:pog}. It is the standard sequential Bayesian {a posteriori} updates of GP regression \eqref{eq:GP_posterior_t} with current dictionary $\bbD_{t}$ operating in tandem with DHMP (Algorithm \ref{alg:komp}), i.e.,
\begin{align}\label{eq:pog}
(\bbmu_{\bbD_{t+1}},&\bbSigma_{\bbD_{t+1}},\bbD_{t+1})\nonumber \\
&=\textbf{DHMP}(\bbmu_{t+1\given \bbD_{t}},\bbSigma_{t+1 \given \bbD_{t}},\tbD_{t+1},\eps_t).
\end{align}
%
 We denote the Gaussian distribution with mean  $\bbmu_{t+1\given \bbD_{t+1}}$ and covariance $\bbSigma_{t+1 \given \bbD_{t+1}}$ at time $t+1$ as $\rho_{\bbD_{t+1}}$.

The compression budget $\eps_t$ may be chosen to
carefully trade off closeness to the population posterior distribution and model parsimony. This tradeoff is the subject of the subsequent section. {Before shifting focus, we provide a more thorough accounting of computational aspects.}

{
\subsection{Computational Effort}\label{subsec:complexity}
As previously noted, the complexity of implementing the dense GP update \eqref{eq:GP_posterior_t} in the online setting is $\ccalO(p t^3)$, due to the need to invert the Gram matrix of kernel evaluations $\mathbf{K}_t$, where $p$ is the parameter dimension ($\bbx\in\ccalX\subset \reals^p$). Moreover, evaluating the empirical kernel map $\bbk_{\ccalS_t}^T(\bbx_{t+1})$ requires at least $\mathcal{O}(p t)$ operations. 
Existing fixed-memory subspace methods \citep{smola2001sparse,seeger2003fast,williams2001using,solin2014hilbert,mcintire2016sparse,bauer2016understanding,bui2017streaming} reduce this complexity bottleneck in sample size $t$ to $\ccalO(p t M^2 )$, where $M$ is an a priori fixed subspace dimension. One way to select $M$ is through the effective problem dimension $d_{\text{eff}}=\text{Tr}\left(\bbK_t (\bbK_t + \lambda \bbI)^{-1}\right)$, i.e., $M=d_{\text{eff}}$ \citep{zhang2003effective,calandriello2017distributed}, which has given rise to specific Nystr{\"o}m sampling schemes \citep{alaoui2015fast,yang2017randomized}. However, doing so supposes one may pre-compute the kernel matrix $\bbK_t$ over $t$ samples, which is inoperable in the online setting. Augmented dimension selection rules based on eigenvalues of $\bbK_t$ have been developed in \citep{calandriello2016analysis,alaoui2015fast}, but again, it is unclear how these statistics may be computed in an incremental manner. \smallskip

{\noindent {\bf Parametric Updates} }By contrast, in this work, the model order $M_t$ is not fixed a priori, but instead adapted sequentially according to \eqref{eq:GP_posterior_D} - \eqref{eq:pog}. This makes the parametric updates per step have $\mathcal{O}(p M_t^3)$ complexity if the kernel matrix inverses are naively computed at each time $t$. However, this may be reduced to  $\mathcal{O}(p M_t^2)$ through using classical rank-one, block-inverse update rules, via the Sherman-Woodbury-Morrison identity -- see \citep{1315946}. This accounts for only the parametric updates, however. \smallskip

{\noindent {\bf Posterior Projections}}  Now we shift to discussing the complexity of computing the pruning phase, i.e., the posterior projections. To do so, denote as $Z_t$ the number of points removed from $\tbD_{t+1}$, and note that $M_{t+1}=M_t$ for $Z_t=1$. By efficiently reusing the kernel matrix inverses as described in the previous step, each parametric evaluation of a GP posterior requires computational
complexity of $\ccalO(p M_t^2)$, and there are $Z_t$ total evaluations of this type. \footnote{A course upper-bound for $Z_t$ is $M_t$, although in practice it tends to be  $2-5$ in transient phase, and is $1$ at stationarity.} Thus, Algorithm \ref{alg:komp} is implementable in $\ccalO(p Z_t M_t^2)$ operations.

 {\noindent {\bf Overall Accounting} } Therefore, the per-step complexity of implementing Algorithm \ref{alg:pog} is $\ccalO(p Z_t M_t^2)$, which when run for $N$ total steps requires $\ccalO(p N Z_t M_t^2)$. Via Theorem \ref{theorem:pog_parsimony}, then, $M_t$ may be replaced in the preceding expression as $\ccalO( N Z_t p (1/\epsilon)^{2p})$, which at steady state $Z_t=1$ simplifies to $\ccalO( N  p (1/\epsilon)^{2p})$. Note this quantity does not require any oracle access to statistics of kernel matrices in advance of running.

}


\section{Balancing Stationarity and Parsimony}\label{sec:convergence}

The foundation of our technical results are the well-developed history of convergence of the empirical posterior \eqref{eq:GP_posterior_t} to the true population posterior \citep{barron1999consistency}. Various posterior contraction rates are available, but they depend on the choice of prior, the underlying smoothness of the generative process $f$, the choice of metric \citep{van2008rates,van2009adaptive,kruijer2010adaptive}, the sampling coverage and radius of the feature space \citep{stuart2018posterior}, among other technicalities. To avoid descending down the rabbit hole of measure theory and functional analysis, we opt for the simplest technical setting we could find for asymptotic posterior consistency of \eqref{eq:GP_posterior_t}, which is stated next.
%
\begin{lemma}\label{lemma:posterior_consistency}\citep{choi2007posterior}[Theorem 6(2) and Sec. 6, Ex. 2] Assume that $f_0$ is the true response function, $\sigma_0^2$ is the true noise variance that corrupts observations $\{y_t\}_{t=0}^\infty$, and $\rho_0$ is the associated true population posterior. Assume the following conditions:
\begin{enumerate}
\item  Suppose training examples $\bbx_t\in\ccalX\subset\reals^p$ are sampled from $\ccalX=[0,1]^p$, and $\{\bbx_t, \bby_t\}_{t=1}^\infty \overset{i.i.d.}{\sim} \mathbb{P}_0$, where $\mathbb{P}_0$ is the true joint distribution of $(\bbx,\bby)$. \label{as:iid}
%
%
%
\item Denote $\lambda(\ccalA)$  as the Lebesgue measure on the hypercube {$\ccalA$ in $[0,1]^p$}. There exists a constant $0<K_p\leq 1$ such that whenever $\lambda(\ccalA)\geq \frac{1}{K_p t}$, {$\ccalA$} contains at least one sample $\bbx_t$.
\item For $t\geq 1$, kernel matrix $\bbK_t$ is positive definite $\bbK_t \succ \bb0$. \label{as:positive_definite}
\item The kernel $\kappa(\bbx, \bbx')$ is of the form $\kappa(\beta\|\bbx - \bbx'\|)$ for some strictly positive $\beta>0$. \label{as:kernel_form}
\item For each $t\geq 1$, there exist positive constants $0<\delta<1/2$ and $b_1,b_2 > 0$, such that the covariance kernel parameter $\beta$ satisfies 
$\mathbb{P}  \{ { \beta > t^{\delta}} \} < b_1 e^{-b_2 t} $ . \label{as:kernel_hyperparameter}
\end{enumerate}
Then, for every $\gamma>0$, the posterior $\rho_t$ defined by \eqref{eq:GP_posterior_t} is asymptotically consistent, i.e.,
\begin{equation}\label{eq:posterior_consistency_uncompressed}
\mathbb{P}_{\Pi}\left\{ d_H(\rho_t,\rho_0)<\gamma \given \ccalS_t \right\} \rightarrow 1 \text{ a.s. with respect to } \mathbb{P}_0
\end{equation}
where the probability $\mathbb{P}_{\Pi}$ computed in \eqref{eq:posterior_consistency_uncompressed} is the Gaussian prior density $\Pi$ as in Sec. \ref{sec:prob}. 
\end{lemma}

This result forms the foundation of our Lyapunov-style analysis of our Bayesian learning algorithm (Algorithm \ref{alg:pog}). {Regarding the required conditions for Lemma \ref{lemma:posterior_consistency}: (i) the assumption that data is i.i.d. from a time-invariant distribution is standard in nonparametric statistics, and is satisfied whenever the underlying generating process is at steady state; (ii) this condition ensures that within a given set of positive measure there exists at least one sample, and holds for any generating process that has probability uniformly bounded away from zero; (iii) ensures that kernel matrices are positive definite, which is true so long as samples are observed without repetition and satisfy (ii); (iv) most covariance kernels take this form, such as the Gaussian, sinc, polynomial, among others; (v) imposes conditions on the rate of attenuation of kernel hyperparameters with the time-index, which is standard for establishing consistency in nonparametric regression \citep{barron1999consistency}.
 }
 
 {Before continuing, we define a looser notion of posterior consistency, which is that the difference between sequential marginal posterior distributions evaluated at sequential inputs $(\bbx_t, y_t)$.

\begin{definition}(Posterior Stationarity)
For fixed prior $\Pi$, a sequence of posterior distributions $\{\rho_t(\cdot) \}_{t\geq 0}$ is a posterior stationary point with respect to a metric $d$ over distributions if the distance between its subsequent marginals evaluated at the sequence $(\bbx_t, y_t)$ approaches null $d(\rho_t, \rho_{t-1}) \rightarrow 0$ with probability $1$, where the probability measure is the prior probability $\Pi$.
\end{definition}
} 
{The motivation and intuition of this definition is that it is a distributional analogue for sequential posterior inference to the notion of a stationary point in global (nonconvex) continuous optimization \citep{nesterov2013introductory} and locally stable equilibria in dynamical systems \citep{wiggins2003introduction}.}
To establish our main result, we additionally require the following assumption, which says that the per step change in the posterior likelihood is unchanged by compression. 

\begin{assumption}\label{as:unbiasedness} 
Define the events $\eta_t:=\{d_H(\rho_t,\rho_{t-1})<\gamma \given \ccalS_t \}$  and $\tilde{\eta}_t:=\{ d_H(\rho_{\tbD_t},\rho_{\bbD_{t-1}}) <\gamma \given \ccalS_t \}$ for any constant $\gamma>0$. The single step likelihood change with respect to the Gaussian prior (Sec. \ref{sec:prob}, paragraph 2) of the uncompressed posterior is at least as likely as the uncompressed single step likelihood change based upon sample point $(\bbx_t,y_t)$ is the same, i.e., 
 $$\mathbb{P}_{\Pi} \{\eta_t  \} \geq \mathbb{P}_{\Pi} \{ \tilde{\eta}_t   \}.$$
 where $\mathbb{P}_{\Pi}$ denotes the prior Gaussian likelihood in Sec. \ref{sec:prob}.
\end{assumption}

Assumption \ref{as:unbiasedness} is reasonable because the uncompressed and compressed updates observe the same sample $(\bbx_t ,\bby_t)$ at time $t$ and formulate conditional Gaussian likelihoods based upon them. While they are conditioned on different dictionaries $\ccalS_{t-1}$ and $\bbD_{t-1}$, the likelihood of the fully dense posterior is at least as likely as the sparse GP. In the analysis, Assumption \ref{as:unbiasedness} plays the role of a Bayesian analogue of nonexpansiveness of projection operators.
Under this assumption and the conditions of the aforementioned lemma, we can establish almost sure convergence under both diminishing and constant compression budget selections as stated next.

\begin{theorem}\label{theorem:pog_consistency}
Under the same conditions as Lemma \ref{lemma:posterior_consistency}, Algorithm \ref{alg:pog} attains the following {posterior stationarity} results almost surely: 
\begin{enumerate}
\item for decreasing compression budget  $\eps_t \rightarrow 0$, for any $\alpha >0$, as $t\rightarrow \infty$, we have $\mathbb{P}_{\Pi}\{ d_H(\rho_{\bbD_t},\rho_{\bbD_{t-1}})  < \alpha \given \ccalS_t \} \rightarrow 1$ w.r.t. population posterior $\mathbb{P}_0$.\label{theorem:diminishing}
\item for fixed budget $\eps_t = \eps > 0$, $\epsilon$-approximate convergence with respect to the Hellinger metric is attained, i.e., for any $\gamma >0$, as $t\rightarrow \infty$, $\mathbb{P}_{\Pi}\{ d_H(\rho_{\bbD_t},\rho_{\bbD_{t-1}})  < \gamma +\eps \given \ccalS_t \} \rightarrow 1$ w.r.t. the population posterior $\mathbb{P}_0$.\label{theorem:constant}
\end{enumerate}
\end{theorem}

\begin{proof} See Appendix \ref{sec:apx_thm1}. \end{proof}

Theorem \ref{theorem:pog_consistency} establishes a formal tradeoff between the choice of compression budget and the accuracy with which we are able to lock onto the population posterior distribution. Specifically, for attenuating compression budget, the algorithm retains more and more sample points as time progresses, such that in the limit it exactly {converges to stationarity}. On the other hand, for constant compression budget, we can converge to a neighborhood of {stationarity}, but for this selection we can additionally guarantee that the complexity of the distribution's parameterization never grows out of control. This finite memory property is formalized in the following theorem.

\begin{theorem}\label{theorem:pog_parsimony}
Suppose Algorithm \ref{alg:pog} is run with constant compression budget $\eps>0$. Then the model order $M_t$ of the posterior distributions $\rho_{\bbD_t}$ remains finite for all $t$, and subsequently the limiting distribution $\rho_{\infty}$ has finite model complexity $M^\infty$. Moreover, $M_t \leq M^{\infty}$ for all $t$. More specifically, we have the following relationship between the model complexity, the compression budget, and the parameter dimension $p$:
$$M_t\leq \ccalO\left(\frac{1}{\epsilon}\right)^p \text{ for all } t$$
\end{theorem}
 
 \begin{proof} See Appendix \ref{sec:apx_thm2}. \end{proof}
 
Theorem \ref{theorem:pog_parsimony} establishes a unique result for approximate GPs, namely, that the attainment of approximate {stationarity} comes with the salient feature that the posterior admits a parsimonious representation. In the worst case, the model complexity depends on the metric entropy of the feature space $\ccalX$, rather than growing unbounded with the iteration index $t$. The combination of Theorems \ref{theorem:pog_consistency}\ref{theorem:constant} and \ref{theorem:pog_parsimony} establish a rate distortion theorem for GPs over compact features spaces {akin to classically related ideas from information theory \citep{cover1999elements}}. In the subsequent section we empirically validate the aforementioned theoretical results on real data sets.

{Extending these results to the use of inducing inputs online is technically challenging due to the temporal correlation they introduce in the observation sequence. That is, when we move away from the kernel dictionary elements being defined differently from the training examples themselves, i.e., via inducing input search \citep{snelson2006sparse,titsias2009variational,bui2017streaming},  i.i.d. assumptions on the observation sequence $(\bbx_t,y_t)$ break down. Unfortunately, most posterior consistency results break down as well for this case. This is why \citep{burt2019rates} focuses on sparse variational GPs when inducing inputs are \emph{fixed}. 

Efforts to go beyond i.i.d. data do exist: when the rate of correlation, i.e., mixing rates \citep{yu1994rates}, decreases \emph{exponentially fast}, consistency of the mean of the GP is possible. More specifically, by exploiting the link between the mean of a GP to kernel ridge regression \citep{van2008reproducing}, and convergence rates for kernel ridge regression on mixing sequences \citep{xu2008learning,steinwart2009fast}, it may be possible to derive conditions on how much inducing inputs are permitted to move per update, which under vanilla GP regression do not move at all. However, the magnitude of their change must be proportional to the mixing rate by suitably specifying the step-size. Unfortunately, this step-size likely must decrease \emph{exponentially fast.} This means that asymptotically (i) the data is i.i.d. and (ii) there are no updates to inducing inputs. For these reasons, the theoretical study of online GP regression with inducing inputs is an open problem.}

{\section{Numerical Validation}\label{sec:pog_consisitency}
In this section, we numerically validate the consistency of POG (Algorithm \ref{alg:pog}). We numerically and statistically observe the fact that the posterior associated with POG is consistent, i.e., that approximates the true distribution. To do so, we consider two examples: the observed data sequence is generated from: (i) Gaussian distribution; (ii) Non-central Chi-square distribution. Next, we present the results related to the above two cases.

\subsubsection{Gaussian distribution}
We begin by considering the simple setting of fixed prior on the feature vector, $\bbx \sim \ccalN( \bb0,\bbK_N)$ and linear observation model $y=\bba^T \bbx + b$, where $\bba$ and $b$ are constants. This implies that $y$ is Gaussian distributed with time-invariant mean $b$ and variance $\bba^T\bba$. We predict the mean and covariance of the sampled data and show that POG follows closely the distribution of the training data. We consider a sequence of  $N=2200$ training examples according to this observation model, run hyperparameter optimization  on a validation set of size $200$ using MATLAB's inbuilt function (``fmincon''), and fix the compression budget as $\eps_t=10^{-4}$. For the experiments, the dimension of the feature vector is set at $10$, and the constant $\bba$ is a $10$ dimensional vector with each value set at $0.5$, and the constant $b$ is taken as $2$.  

 Since the resulting data is Gaussian distributed, we can validate the consistency of the posterior inference as compared with the ground truth using a $z-test$ \citep{sprinthall1990basic,casella2002statistical} at significance level $\alpha=0.05$.
The posterior parameters of POG yielded a $z-test$ whose $p$-value is near 0, which establishes that we may be at least $95\%$ confident of  our claim that POG adheres to the true posterior. This $z-test$  was carried out was carried on the predicted mean of last $100$ samples of training data. The probability that the hypothesis is satisfied obtained from the test is $0.9250.$

The same can be also be validated from the plots of Fig. ~\ref{fig:pog_consistency_lin}. From Fig.~\ref{fig:hist_train_lin}, it can be observed that the histogram of posterior mean obtained from the POG algorithm very closely matches with the histogram of the training data. Thus substantiating our claim that POG closely follows the distribution of the training data, which in this case is Gaussian distribution. In Fig. ~\ref{fig:Mean_comparison_lin}, we compare the mean of the training samples obtained from the POG algorithm with the actual mean of the training samples, and it can be observed that POG closely approximates the mean. Similarly, we compare the variances in Fig. ~\ref{fig:var_comaprison_lin}. The above performance is observed with a final settled model order complexity of 127 verifying  the finite model order growth property of POG algorithm. To average out the noisy estimates of mean and variance and for better interpretability of results, we have considered the moving average of 200 samples. Since, the point estimates are noisy, thus they are greyed out in Fig.~\ref{fig:Mean_comparison_lin} and Fig.~\ref{fig:var_comaprison_lin}.

\begin{figure*}[t!]
\begin{subfigure}[t]{0.33\linewidth}
\centering
\includegraphics[width=\linewidth,height=0.8\linewidth]{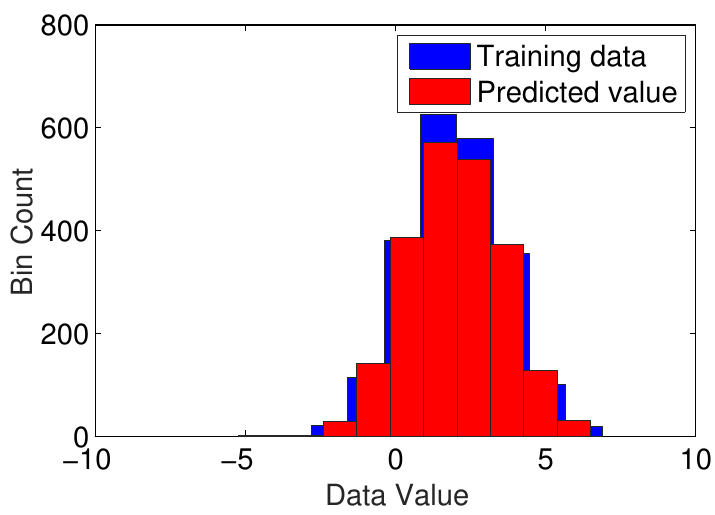}
\caption{}
\label{fig:hist_train_lin}
\end{subfigure}
\begin{subfigure}[t]{0.33\linewidth}
\centering
\includegraphics[width=\linewidth,height=0.8\linewidth]{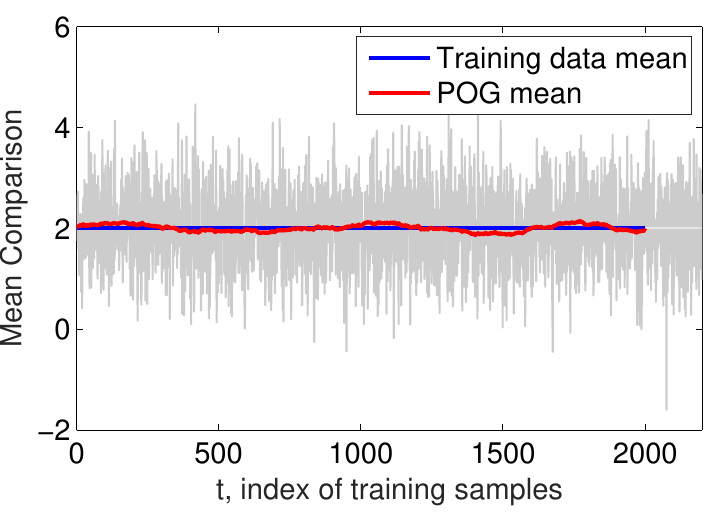}
\caption{}
\label{fig:Mean_comparison_lin}
\end{subfigure}
\begin{subfigure}[t]{0.33\linewidth}
\centering
\includegraphics[width=\linewidth,height=0.8\linewidth]
{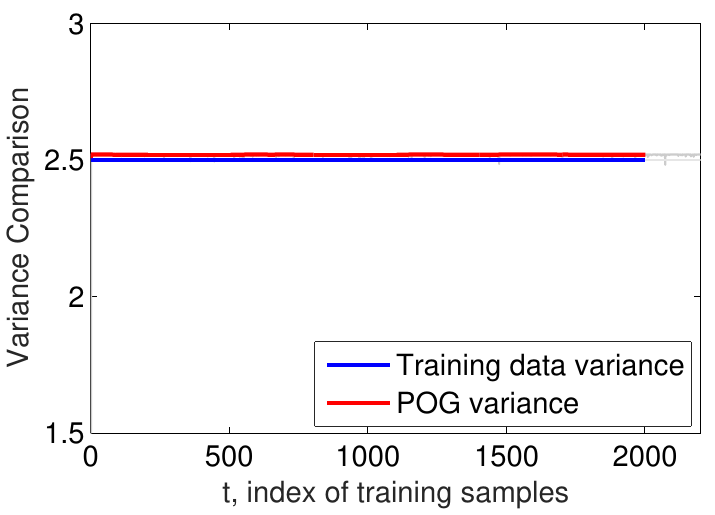}
\caption{}
\label{fig:var_comaprison_lin}
\end{subfigure}
\vspace*{-0.2cm}
\caption{Performance of POG on a Gaussian distributed data set. POG very closely follows the mean and variance of the distribution of the data set, i.e., Gaussian distribution.  \vspace{-5mm}}
\label{fig:pog_consistency_lin}
\end{figure*}

\subsubsection{Chi-square distribution}
Next, we evaluate the consistency of POG algorithm considering data obtained from a chi-square distribution with degree of freedom $1$. The feature vector, $x_t \sim \ccalN(\omega t^{-k},1)$, where $\omega$ is a constant. The target variable is generated as $y_t=x_t^2$. Thus the true distribution of $y_t$ is a non-central Chi-squared distribution with one degree of freedom, and mean $1+(\omega t^{-k})^2$ and variance $2[1 + 2(\omega t^{-k})^2]$. For our experiments, we have considered $\omega=1$, and $k=1/10$, a training data set of size $2200$. For hyperparameter selection, we have considered a validation set of $200$ samples for hyperparameter optimization via MATLAB's inbuilt function ("fmincon"). The compression budget is set at $\eps_t=10^{-5}$. 

Similar to the previous subsection, here also we get equivalent results verifying the consistency of POG algorithm. In Fig. ~\ref{fig:hist_train_chi_2}, it can be clearly observed that POG very well approximates the training data distribution. Here also, we have considered the moving average of 200 samples and have greyed out the noisy point estimates. From Fig. ~\ref{fig:Mean_comparison_chi_2} and Fig. ~\ref{fig:var_comaprison_chi_2} it can be observed that POG very closely follows the mean and variance of the true distribution. Thus validating the consistency of the POG algorithm in approximating the true distribution.

\begin{figure*}[t!]
\begin{subfigure}[t]{0.33\linewidth}
\centering
\includegraphics[width=\linewidth,height=0.8\linewidth]{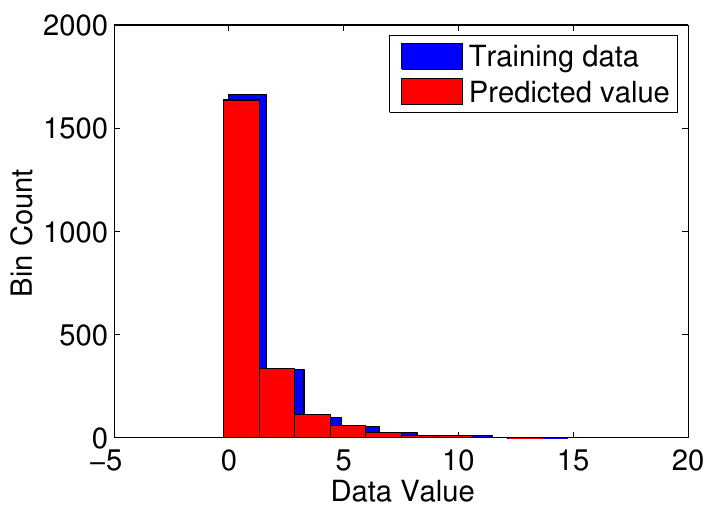}
\caption{Histogram of training data}
\label{fig:hist_train_chi_2}
\end{subfigure}
\begin{subfigure}[t]{0.33\linewidth}
\centering
\includegraphics[width=\linewidth,height=0.8\linewidth]{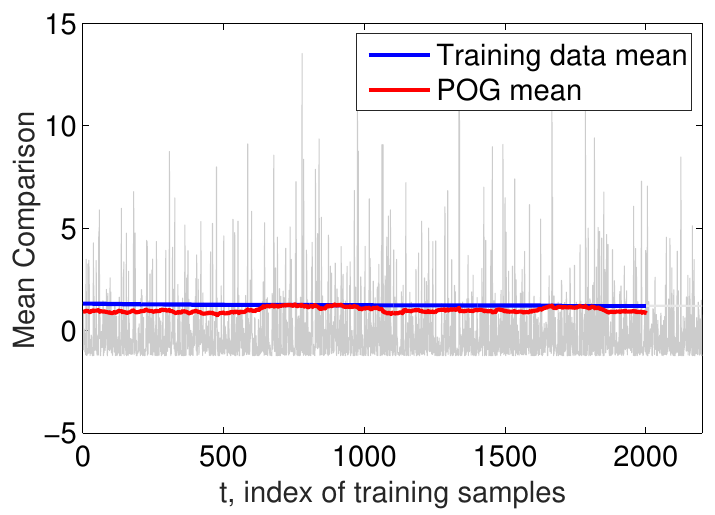}
\caption{}
\label{fig:Mean_comparison_chi_2}
\end{subfigure}
\begin{subfigure}[t]{0.33\linewidth}
\centering
\includegraphics[width=\linewidth,height=0.8\linewidth]
{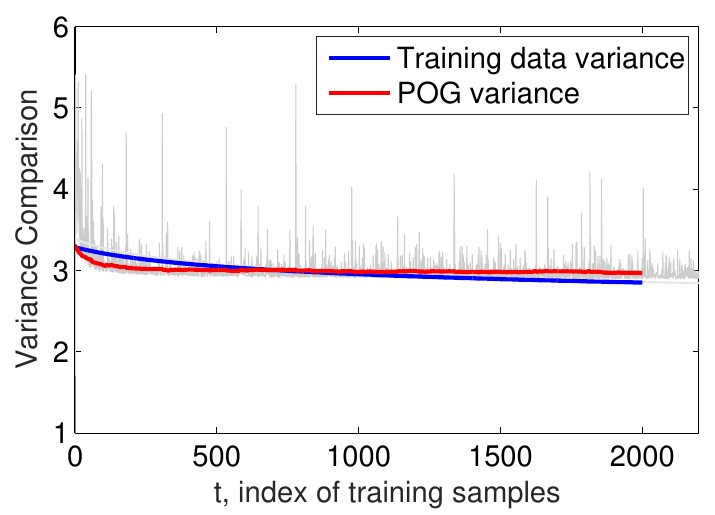}
\caption{}
\label{fig:var_comaprison_chi_2}
\end{subfigure}
\vspace*{-0.2cm}
\caption{Performance of POG on non-central Chi-squared distributed (one degree of freedom) data set where the feature vector,  $x_t \sim \ccalN(\omega t^{-1/10},1)$. POG tracks the variation of the data set by closely approximating the mean and variance.\vspace{-5mm}}
\label{fig:pog_consistency_chi_2}
\end{figure*}
}


\section{Experiments}\label{sec:experiments}
In this section, we compare the performance of POG with standard online and offline approximations of GP: Sparse Online Gaussian Processes (SOGP)\citep{csato2002sparse}, which fixes the subspace dimension a priori, and an offline fixed-dimension approach which additionally does a gradient-based search over the space of training examples dubbed pseudo-input search (Pseudo-Input Gaussian Processes, or SPGP) \citep{snelson2006sparse}. Along with the above three algorithms, we also consider the performance of the Dense GP, which stores all training samples and does no compression. For our experiments, the Dense GP is implemented in an online fashion for the purpose of visualization, but in practice this is impossible as when the size of training set becomes large. Thus, we report the limiting test error as an offline approach in Table \ref{table:algorithm comparison}.

 To ensure competitiveness of each approach, we conduct hyperparameter optimization at the outset over a number of different bandwidth selections using MATLAB's inbuilt function ("fmincon") for POG and Dense GP, whereas for SOGP and SPGP we employ their in-built hyper-parameter selection schemes { on respective validation sets of size $300$ (for kin40k data), $150$ (for abalone data) and $50$ (for boston housing data), which are then held constant for the whole duration of training.}
  We note that since the input dimension defines a covariance component for each element of the kernel, rather than use a single kernel bandwidth, we tune a diagonal of matrix of bandwidth parameters often referred to in the  literature as ``automatic relevance determination." These parameters are tuned over a randomly selected small subset of the training samples. An analogous procedure was used to tune the noise prior in experiments.

{Online augmentations of pseudo-input approximations exist \citep{snelson2006sparse} which \emph{adapt} inducing-points according the variational approximation to the GP likelihood, as in \citep{titsias2009variational,bauer2016understanding,bui2017streaming}, when their total the number is a priori fixed. Doing so, however, designs the GP posterior along a separate criterion than that which is the emphasis of this work. In particular, the goal of this section is to rigorously analyze the role of \emph{point selection}, \emph{not} point adaptation.}

For our experiments, we consider the Gaussian kernel
with varied length-scale, i.e., $\kappa(\bbx_m,\bbx_n)=$\\ $a\exp\Big[-\frac{1}{2}\sum_{i=1}^p\frac{(\bbx_m^{(i)}-\bbx_n^{(i)})^2}{{\bbq^{(i)}}^2}\Big]$ with $\theta=\{a,\bbq\}$ as hyperparameters, where the $i$th superscript in the variable denotes the $i$th component of the vector.

  
  POG being an online algorithm, for every training instant $t$, we run the Algorithm \ref{alg:pog}  and use the dictionary $\bbD_t$ obtained from the compression Algorithm \ref{alg:komp} to evaluate the following measures: the standardized mean squared error (SMSE) and the mean standardized log loss (MSLL), or  on the test data set using \eqref{eq:GP_posterior_D}, i.e.,
  \begin{align}\label{eq:GP_posterior_test_data}
\bbmu^{test}_{i\given \bbD_{t} }&=\bbk_{\bbD_{t}}(\bbx^{test}_{i})[\bbK_{\bbD_{t},\bbD_{t}}+ \sigma^2 \bbI]^{-1}\bby_t  \\
		\bbSigma^{test}_{i\given \bbD_{t}}&\!\!=\!\kappa(\bbx^{test}_{i},\bbx^{test}_{i}) \nonumber\\
		 &~~- \bbk^T_{\bbD_{t}}\!(\!\bbx^{test}_{i}\!) [\bbK_{\bbD_{t},\bbD_{t}}\!\!+\! \sigma^2 \bbI]^{-1}\bbk_{\bbD_{t}}(\bbx^{test}_{i})+\sigma^2, \nonumber
\end{align}
 where $i$ denotes the index of test data samples and $\bbx^{test}_{i}$ is the $i$th sample from the test data set.  The index $i$ varies from $i=1,\dots,N_{test}$, where $N_{test}$ is the total number of samples in the test data set. The specific way SMSE and MSLL are calculated on the test  set for every training index  $t$ is given below:
  \begin{align}\label{eq:smse_msll}
	 SMSE &=  \frac{1}{N_{test}}\sum_{i=1}^{N_{test}}\frac{\Big(y^{test}_{i}-\bbmu^{test}_{i\given \bbD_{t} }\Big)^2}{var[\bby]} \\
	\!\! MSLL&\!= \frac{1}{2N_{test}}\!\!\sum_{i=1}^{N_{test}}\!\!\!\Bigg(\!\frac{\Big(y^{test}_{i}\!-\!\bbmu^{test}_{i\given \bbD_{t} }\Big)^2}{\bbSigma^{test}_{i\given \bbD_{t}}}\! + \! \log\!\Big(|\bbSigma^{test}_{i\given \bbD_{t}}|\Big)\!\!\!\Bigg),\nonumber
	\end{align}
	where $var[\bby]$ is the variance of the training data and $y^{test}_{i}$ is the actual test value. Observe that MSLL is just the negative log likelihood with a constant of $2\pi$ omitted, since this term does not reflect the accuracy of mean and covariance estimates. Thus, smaller values of MSLL are better in the sense of maximum a posteriori estimation.   We next use SMSE and MSLL for the performance evaluation of the algorithms on the Boston \citep{harrison1978hedonic}, kin40k data \citep{seeger2003fast}, and Abalone \citep{nash1994population} data sets, to faithfully compare against benchmarks that appeared early in this literature, i.e., \citep{seeger2003fast}.
  
For fair comparison of POG with SOGP and SPGP, we have kept the model order (total number of dictionary points) of all the three algorithms to be same. However, this restriction is infeasible for the Dense GP whose model order grows by one with every training sample. The performance of the Dense GP stands as a performance benchmark.  The exact value of the model order differs for each data set and is reported in {Table \ref{table:model_time_complexity}} {and  the performance in terms of \eqref{eq:smse_msll} is presented in Table. \ref{table:algorithm comparison}}. The values for SMSE and MSLL are obtained by computing their average over the test data set for the last $100$ samples of training set. We categorize approaches as offline and online based upon the tractability of computing the GP posterior on the fly. {The run time of an algorithm is measured using the MATLAB's inbuilt function (``tic'' and ``toc'') and is reported in Table. \ref{table:model_time_complexity}. The average performances and one-standard error bars were obtained by repeating the experiment $7$ times and the same can be observed in the experimental results.}

\begin{table*}[]
\centering
\renewcommand\tabcolsep{8pt}
\caption{Comparison of online (POG, SOGP) and offline (Dense GP and SPGP) approaches on Boston, Abalone and kin40k data set. Performance is quantified in terms of SMSE and MSLL \eqref{eq:smse_msll}. Best performance attained under both online and offline settings is presented in boldface and the number of training samples and test samples are denoted by $N_{trn}$ and $N_{test}$, respectively. For ranking performance, we consider the best SMSE and MSLL achieved under a given model fitness. Observe POG consistently attains superior performance in the online setting, and is often comparable to offline benchmarks.}\label{table:algorithm comparison}
\begin{tabular}{|c|c|c|c|c|c|}
\hline
\multirow{2}{*}{Data set} & \multirow{2}{*}{$N_{trn}/N_{test}$}  & \multicolumn{4}{c|}{Performance of algorithms (SMSE / MSLL)} \\ \cline{3-6}
 &                                                                                 &                                                                          \multicolumn{2}{c|}{Offline}  & \multicolumn{2}{c|}{Online} \\ \cline{3-6}    & &\multirow{2}{*}{\begin{tabular}[c]{@{}c@{}}Dense GP  \end{tabular}}& \multirow{2}{*}{\begin{tabular}[c]{@{}c@{}}Offline Benchmark \\ (SPGP) \end{tabular}}   & POG  & SOGP  \\[10pt] \hline
Boston                    & 455                                                                               /51                                                                                & 0.0989/0.3364/& \bf{0.0833/0.1691}                &    \bf{0.2590/0.6323}         & 0.4629/2.4241  \\[3pt] \hline
Abalone                   & 3133                                                                              /1044                                                                             & 0.4078/2.1529 &   \bf{0.3865
/2.0083}                   &  \bf{0.4324/2.2032}
          & 0.4839/357.4717 \\[3pt] \hline
kin40k                    & 4000                                                                               /200                                                                              & 0.0686/0.0694&  \bf{0.1388/0.4104}       &    \bf{0.1943/0.5620}       & 0.8131/30.5652   \\[3pt] \hline
\end{tabular}
\end{table*}

\begin{table*}[]
\centering
\renewcommand\tabcolsep{8pt}
\caption{{Comparison of Model and Time coplexity on online (POG, SOGP) and offline (Dense GP and SPGP) approaches on Boston, Abalone and kin40k data set. The model complexity is quantified in terms of the model order (M). The time complexity here denotes the total run time of a full training epoch and average time for a single incremental training update in seconds. Since SPGP is an offline algorithm thus the incremental update time is not available. The per-update runtime for POG is slower than SOGP, but faster than the dense GP or SPGP. This is a justifiable tradeoff in terms of the increased accuracy obtained by POG as depicted in Table \ref{table:algorithm comparison}.} }\label{table:model_time_complexity}
{\begin{tabular}{|c|c|c|c|c|}
\hline
\multirow{2}{*}{Data set} & \multicolumn{4}{c|}{ Model complexity (M)/ Time complexity: full training epoch (sec)/ Single incremental update (sec)} \\ \cline{2-5}
 &                                                                                                                                                       \multicolumn{2}{c|}{Offline}  & \multicolumn{2}{c|}{Online} \\ \cline{2-5}     &\multirow{2}{*}{\begin{tabular}[c]{@{}c@{}}Dense GP  \end{tabular}}& \multirow{2}{*}{\begin{tabular}[c]{@{}c@{}}Offline Benchmark \\ (SPGP) \end{tabular}}   & POG  & SOGP  \\[10pt] \hline
Boston                     & 405/{97.48}/0.2407& {83/6.72/-}                &    {83/{54.89}/0.1355
}         & 83/{12.06}/0.0298  \\[3pt] \hline
Abalone       & 2983/{$3.9373\times10^5$}/  131.9913 &   {394/94.32/-}                   &  {394/{$1.5289\times 10^4$}/5.1254}
          & 394/{412.57}/0.1383    \\[3pt] \hline
kin40k               & 3700/{$2.2464\times10^5$}/60.7135 &  {392/104.96/-}       &    {392/{$6.4209\times 10^3$}/1.7354}       & 392/{391.34}/ 0.1058   \\[3pt] \hline
\end{tabular}}
\end{table*}

\begin{figure*}[t!]
\begin{subfigure}[t]{0.49\linewidth}
\centering
\includegraphics[width=\linewidth,height=0.8\linewidth]{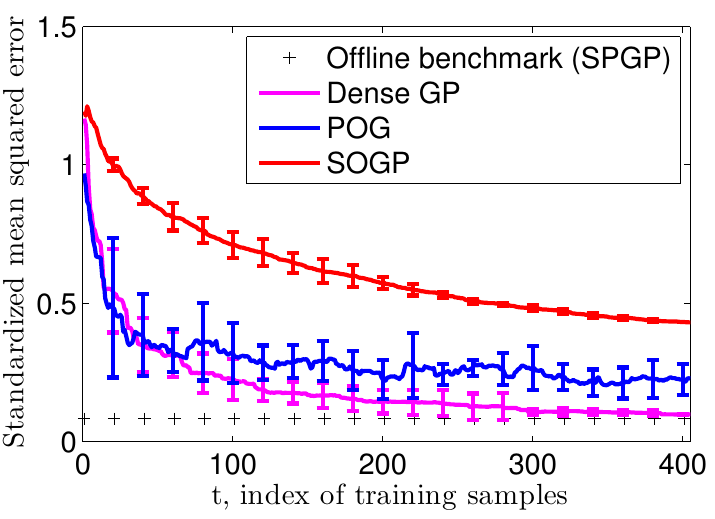}
\caption{}
\label{fig:boston_smse}
\end{subfigure}
\begin{subfigure}[t]{0.49\linewidth}
\centering
\includegraphics[width=\linewidth,height=0.8\linewidth]
{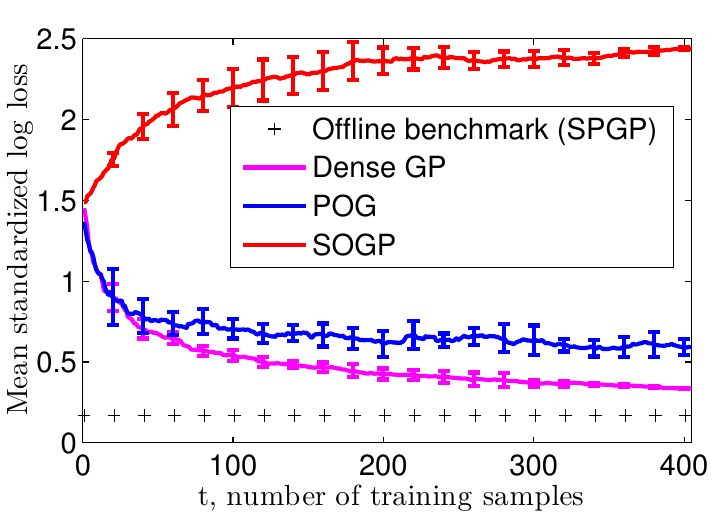}
\caption{}
\label{fig:boston_msll}
\end{subfigure}
\begin{subfigure}[t]{0.49\linewidth}
\centering
\includegraphics[width=\linewidth,height=0.8\linewidth]{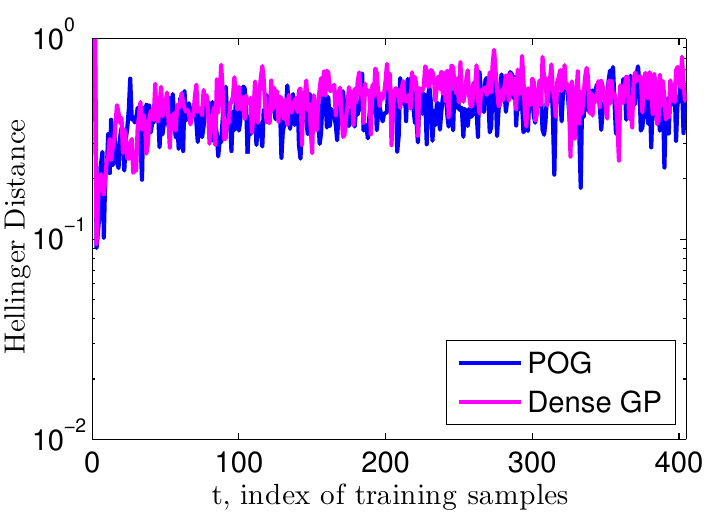}
\caption{}
\label{fig:hellinger_boston}
\end{subfigure}
\begin{subfigure}[t]{0.49\linewidth}
\centering
\includegraphics[width=\linewidth,height=0.8\linewidth]
{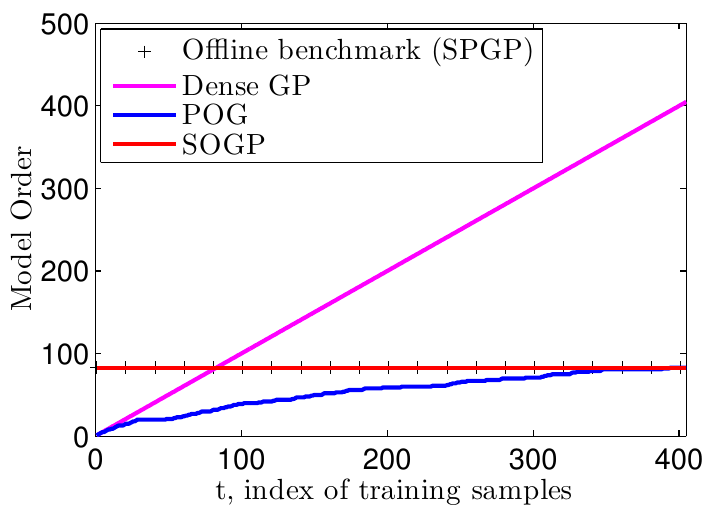}
\caption{}
\label{fig:mo_boston}
\end{subfigure}
\vspace*{-0.2cm}
\caption{Performance of POG, SOGP \citep{csato2002sparse}, SPGP \citep{snelson2006sparse} and Dense GP on \textit{boston} data set. POG outperforms SOGP and yields comparable performance to Dense GP and SPGP, while being able to operate online with a dynamic model order. The limiting model order of POG is independent of the training sample size.\vspace{-5mm}}
\label{fig:boston}
\end{figure*}


\begin{figure*}[t]
\begin{subfigure}[t]{0.49\linewidth}
\centering
\includegraphics[width=\linewidth,height=0.8\linewidth]{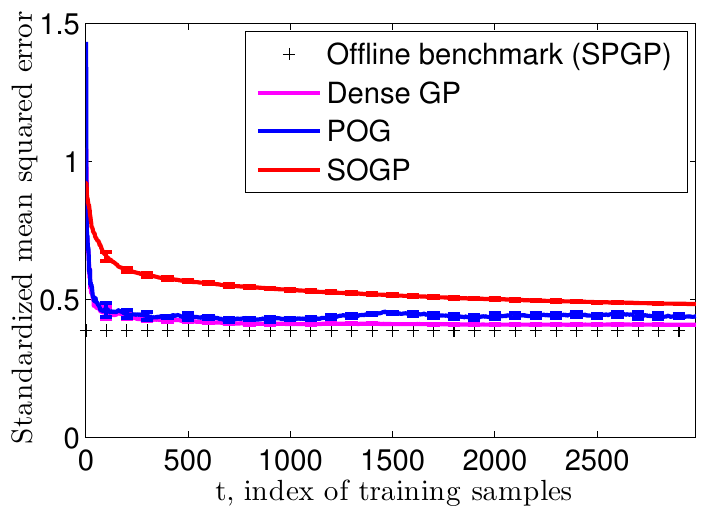}
\caption{}
\label{fig:abalone_smse}
\end{subfigure}
\begin{subfigure}[t]{0.49\linewidth}
\centering
\includegraphics[width=\linewidth,height=0.8\linewidth]
{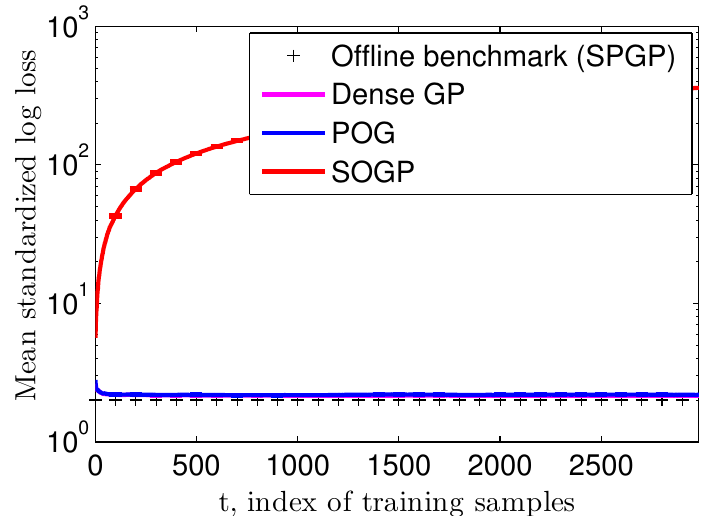}
\caption{}
\label{fig:abalone_msll}
\end{subfigure}

\begin{subfigure}[t]{0.49\linewidth}
\centering
\includegraphics[width=\linewidth,height=0.8\linewidth]{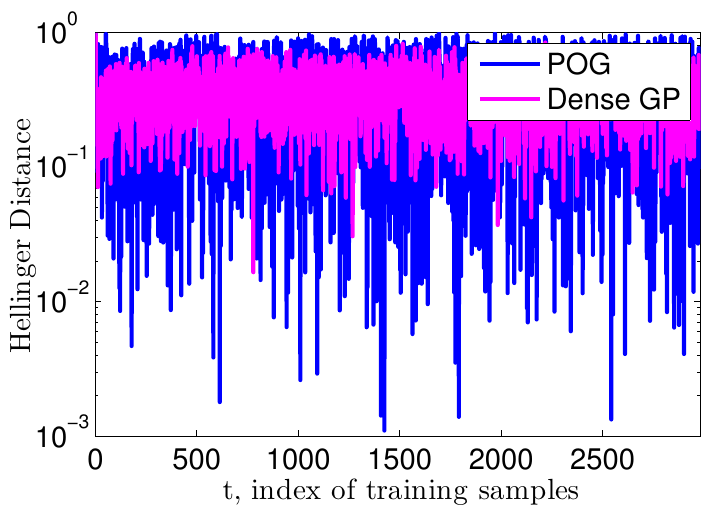}
\caption{}
\label{fig:hellinger_abalone}
\end{subfigure}
\begin{subfigure}[t]{0.49\linewidth}
\centering
\includegraphics[width=\linewidth,height=0.8\linewidth]
{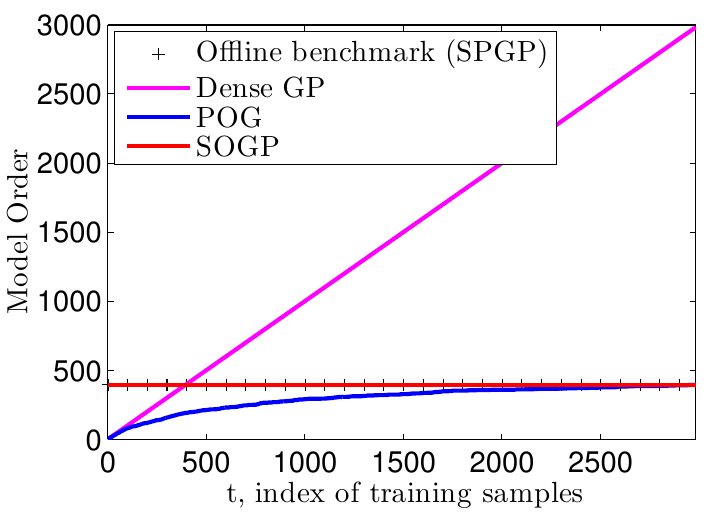}
\caption{}
\label{fig:mo_abalone}
\end{subfigure}
\vspace*{-0.2cm}
\caption{Performance of POG, SOGP \citep{csato2002sparse}, SPGP \citep{snelson2006sparse} and Dense GP on \textit{abalone} data set. POG  outperforms SOGP and yields performance close to Dense GP with significantly fewer model points than the Dense GP. Moreover, POG performs comparable to  SPGP which employs a complicated offline hyper-parameter search.\vspace{-5mm}}
\label{fig:abalone}
\end{figure*}

\begin{figure*}[t]
\begin{subfigure}[b]{0.49\linewidth}
\centering
\includegraphics[width=\linewidth,height=0.8\linewidth]{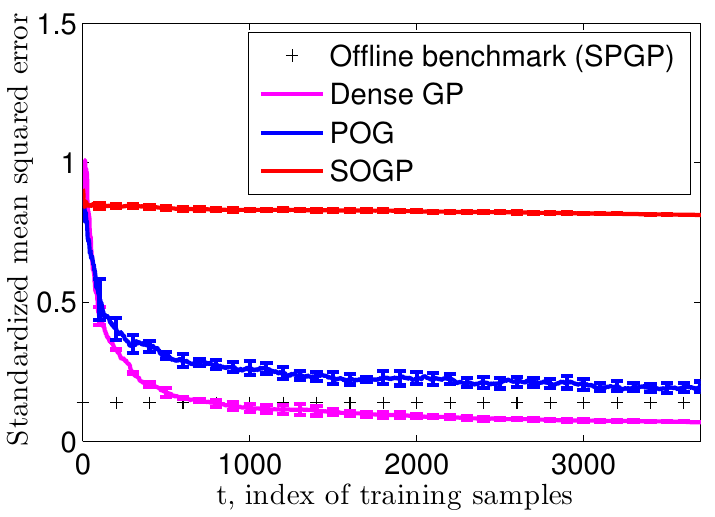}
\caption{}
\label{fig:kin_smse}
\end{subfigure}
\begin{subfigure}[b]{0.49\linewidth}
\centering
\includegraphics[width=\linewidth,height=0.8\linewidth]
{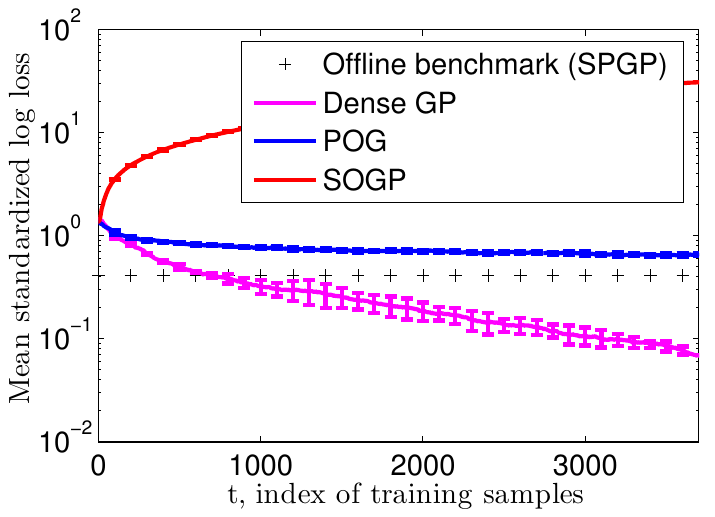}
\caption{}
\label{fig:kin_msll}
\end{subfigure}

\begin{subfigure}[b]{0.49\linewidth}
\centering
\includegraphics[width=\linewidth,height=0.8\linewidth]
{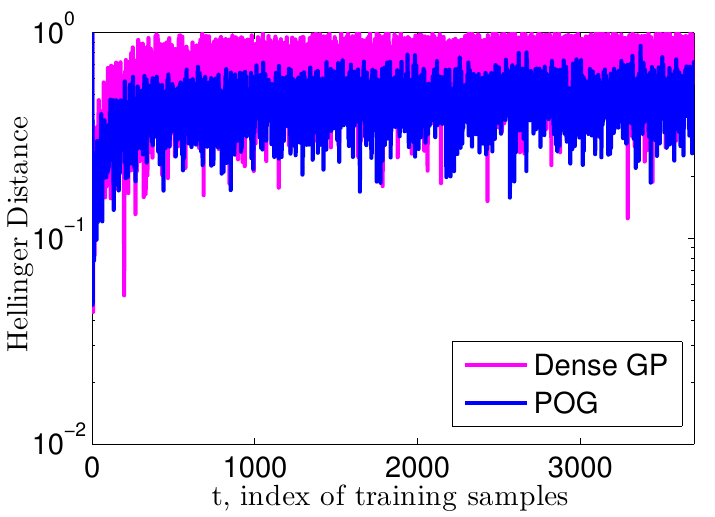}
\caption{}
\label{fig:hellinger_kin}
\end{subfigure}
\begin{subfigure}[b]{0.49\linewidth}
\centering
\includegraphics[width=\linewidth,height=0.8\linewidth]
{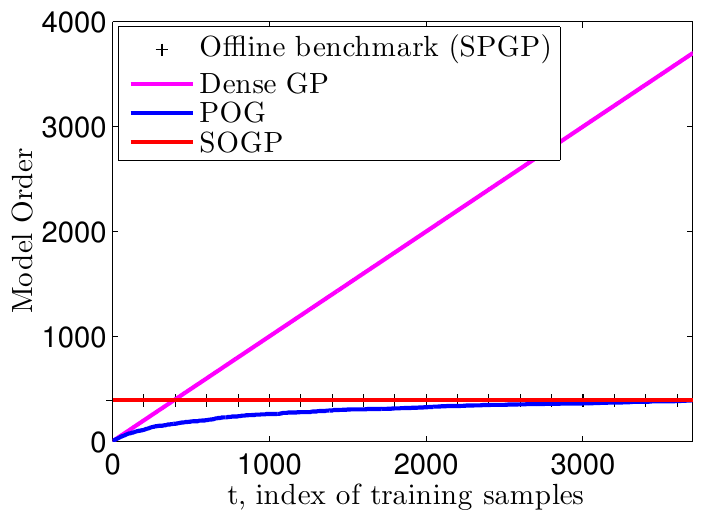}
\caption{}
\label{fig:modelorder_kin}
\end{subfigure}
\caption{Performance of POG, SOGP \citep{csato2002sparse} and SPGP \citep{snelson2006sparse} and Dense GP on \textit{kin-40k} data set. POG yields superior performance to SOGP, and attains comparable model fitness to SPGP and Dense GP, which are only implementable in the offline setting. POG achieves this performance by keeping the model order dynamic and not fixed as SOGP and SPGP, which eventuates in complexity that is independent of the training sample size.\vspace{-5mm}}
\label{fig:kin40k_data}
\end{figure*}

\subsection{Boston Housing data}\label{sec:boston}
In this section we study the performance of POG along with three other algorithms mentioned above on a real data set obtained from housing information in the area of Boston\footnote{https://www.cs.toronto.edu/~delve/data/boston/bostonDetail.html}. There are a total of $455$ training samples and $51$ test samples with an input dimension of size $8$. The compression budget, $\eps_t$ was fixed at $\eps_t=4.9\times 10^{-5}$. 

 For the fair comparison, we have considered the constant model order of SOGP and SPGP to be equal to the final settled model order of POG, i.e., $83$ (can be observed from Fig.\ref{fig:mo_boston}). In Fig. \ref{fig:hellinger_boston} and \ref{fig:mo_boston}, we demonstrate the evolution of POG in terms of the Hellinger distance and model order. {From Fig. \ref{fig:hellinger_boston} it can be observed that the hellinger distance between the successive iterates of POG lower bounds the hellinger distance between successive iterates of Dense GP, thereby validating our Assumption  \ref{as:unbiasedness}.} In Fig. \ref{fig:mo_boston}, we present the evolution of model order of all approaches, from which we may observe that the Dense GP retains all past points into its posterior representation, and hence its complexity grows unbounded. By contrast, the model complexity of POG settles down to $83$. The fixed dimension approach of the SOGP and SPGP algorithm can also be verified from Fig. \ref{fig:mo_boston}.
 
 We compare the performance of POG, SOGP, Dense GP and SPGP in Figs \ref{fig:boston_smse} and \ref{fig:boston_msll}. Observe that for both the test error metrics (SMSE and MSLL) \eqref{eq:smse_msll}, POG outperforms SOGP and gives comparable performance to Dense GP. However, Dense GP achieves better performance at the cost of growing dictionary size which is evident from Fig. \ref{fig:mo_boston}, i.e., $405$ in comparison to the model order of $83$ of POG.  Note that SPGP achieves the best performance, but this performance gain is attained by conducting hyper-parameter (pseudo-input) search over training examples themselves, which is possible by virtue of offline processing. In contrast, POG, SOGP and the Dense GP operate online. This reasoning is why we clarify that SPGP is labeled in the legend as an ``Offline benchmark''. Amongst the online approaches, POG attains a favorable tradeoff of complexity and model fitness.

\subsection{Abalone data} \label{sec:abalone}
We shift focus to studying the performance of POG algorithm on a larger data set \textit{abalone} \citep{nash1994population} comprised of $3133$ training samples and $1044$ test samples relative to the aforementioned comparators. The data set focuses on the task of predicting physical properties of abalone, a type of shellfish. 
 The objective is to determine the age of abalone by predicting the number of rings from the $8$ dimensional inputs vector. The input vector consists of different attributes like sex, length, height, weight and other weights. The hyperparameter optimisation and noise prior is obtained exactly in the same way as explained before in Sec. \ref{sec:experiments}. The compression budget, $\eps_t$ is fixed at $1\times 10^{-5}$. 

Similar to the \textit{Boston} data set, here also we have kept the model order of the SOGP and SPGP constant at $394$ (the final settled model order of POG) for fair comparison of both the algorithms. From Fig. \ref{fig:abalone_smse} and Fig. \ref{fig:abalone_msll}, it can be observed that POG performs better than SOGP and yields performance close to Dense GP and SPGP for both the test error metrics with a model order growth directly determined by the statistical significance of training points across time (Fig. \ref{fig:mo_abalone}). The evolution of Hellinger distance is plotted in Fig. \ref{fig:hellinger_abalone}. From Fig.\ref{fig:mo_abalone}, one may observe that although 2983 training samples have been processed, POG determines to retain only $394$ dictionary points, in contrast to the full dictionary of 2983 points in the ``Dense GP''. Thus, overall one may conclude that POG attains comparable model fitness to the Dense GP and SPGP while being amenable to online operation, in contrast to SOPG. 

\subsection{kin40k data}
Next we study the performance of POG on the \textit{kin-40k} data set. This is a data set belonging to the \textit{kin} family of datasets from the DELVE archive -- see \citep{seeger2003fast}. The data set is generated from the realistic simulation of the forward dynamics of an 8 link all-revolute robot arm. We consider a small subset of 4000 training samples and 200 test samples from the original data set for our experiment. The goal is to estimate the distance of the robot arm head from a target, based on the 8 dimensional input consisting of joint positions and twist angles\footnote{http://www.cs.toronto.edu/~delve/data/kin/desc.html}. The \textit{kin-40k} data set is generated with maximum non-linearity and little noise, there by resulting in a highly difficult regression task\footnote{http://ida.first.fraunhofer.de/anton/data.html} \citep{chen2006trends}.

 The performance of the POG algorithm on the $\textit{kin-40k}$ data set is plotted in Fig. \ref{fig:kin40k_data}.  The compression budget, $\eps_t$ was fixed at $\eps_t=1\times10^{-5}$. Observe that in Fig. \ref{fig:kin_smse} and Fig. \ref{fig:kin_msll} POG performs better than SOGP and gives comparable performance to SPGP and Dense GP if not better. The model order $M$ for SPGP and SOGP is fixed at $392$, the complexity discerned by POG in Fig. \ref{fig:modelorder_kin}. This is in contrast to the linear increase of the Dense GP. We further visualize the distributional evolution of POG in Fig. \ref{fig:hellinger_kin}. 
 
In contrast to the better performance of SPGP in comparison to the Dense GP algorithm for the previous two data sets, i.e., \textit{Boston} and \textit{Abalone}, here we see the reverse. This may be an artifact of the fact that we considered a subset of data (4000 training samples from 10000 training samples, 200 test samples from 30000 test samples) from the actual kin40k data set for experimentation. Thus, in the case when one operates with a data subset, SPGP is not able to optimize pseudo-inputs to be representative of the subset, whereas Dense GP which stores all the points in the dictionary, and thus performs favorably. This phenomenon inverts the ranking observed in the \textit{Boston} and {Abalone} data sets, and suggests pseudo-input search of SPGP must happen together with selecting the appropriate number of points, which in practice is challenging.

{One additional note is that across these distinct data domains, one may observe that SPGP outperforms the dense GP. As pointed out by \citep{burt2019rates}, this may be attributable to the fact that a fixed offline pseudo-input dimension fits a model along only the dominant eigenvalues of the associated kernel matrix. For cases where the covariance of the underlying generating process exhibits comparable eigenvalue concentration, this can prove beneficial.}

{Next, we present the run time (in seconds) for a full training epoch and a single incremental update of all the algorithms in Table. \ref{table:model_time_complexity} for all the three data sets.  It can be observed from Table. \ref{table:model_time_complexity} that POG  has significantly more run time as compared to SOGP and SPGP, thus it stands out crucial to bring the key facts pertaining to this and is presented below: 
\begin{itemize}
\item {\bf POG vs. SOGP}: In POG, the model complexity (model order or size of basis vector set) adapts dynamically with the training process. However, in SOGP the model complexity is fixed before the start of the training process, which in various instances is difficult to predict apriori. POG at every training iteration searches for the  optimal number of points which best represent the function. Thus having significantly better performance  which can be observed in the SMSE and MSLL performances from the numerical results and also from Table. 2 of the main paper. 
\item {\bf POG vs. SPGP}: Its needless to say that SPGP being an offline algorithm has all the training data beforehand and thus requires very less run time for the training process. POG being an online algorithm goes through the training process at every iteration and thereby incurring more time than SPGP. 
\item {\bf POG vs. Dense GP}: Since POG works in tandem with the DHMP (Algorithm 2) compression algorithm, thus it learns the optimal number of points required to represent the function with a given error tolerance. However, Dense GP does no compression and keeps on storing all the training points which accounts for better performance but also on the other hand incurs more running time than POG.  
\end{itemize}}


\section{Conclusion}\label{sec:conclusion}

We presented a new approach to sparse approximation to GPs under streaming settings. In particular the computational complexity of the mean and covariance are proportional to time, which renders standard GPs inoperable in online settings. 

To ameliorate this issue, we proposed a sparsification routine which operates on a faster time-scale than Bayesian inference, and  approximates the posterior distribution. We consider Hellinger distance between the current posterior and its population counterpart for studying the theoretical guarantees of our sparse GP algorithm. Since Hellinger distance is computable in closed form for Gaussians, thus making it a suitable choice among valid Lyapunov functions considered for study of asymptotic behavior of GPs. By introducing hard-thresholding projections based on matching pursuit, we were able to design sparsification rules that nearly preserve the theoretical statistical behavior of GPs while bringing their memory usage under control.

 The performance of the algorithm was tested on three highly non-linear real data sets and also compared against benchmark alternatives. The experimentation illuminates the benefits of not fixing the subspace dimension a priori, but instead the statistical error, under practical settings, and supports our theoretical findings. Future directions include the merits and drawbacks of variable-subspace approaches to variational approximations to the GP likelihood, as well as a better conceptual understanding of methods based on pseudo-input search.


\appendix
\section{Technical Lemmas}
\begin{lemma}\label{lemma:lyapunov}
Under the same conditions as Lemma \ref{lemma:posterior_consistency}, the probability of event
$\eta_t:=\{d_H(\rho_t,\rho_{t-1})<\gamma \given \ccalS_t \}$ with respect to $\mathbb{P}_{\Pi}$ approaches $1$ as $t\rightarrow \infty$ for any $\gamma > 0$.
\end{lemma}
%
\begin{proof}
%
%
Let's analyze the probability of event $\eta_t=\{d_H(\rho_t,\rho_{t-1})<\gamma \given \ccalS_t \}$ with respect to $\mathbb{P}_{\Pi}$ by first looking at the argument of the event, $d_H(\rho_t,\rho_{t-1})$, and applying the triangle inequality
\begin{equation}\label{eq:triangle}
d_H(\rho_t,\rho_{t-1}) \leq d_H(\rho_t,\rho_{0}) + d_H(\rho_{t-1},\rho_{0}) 
\end{equation}
From posterior contraction rates \citep{ghosal2000convergence}[Theorem 2.5] we have $d_H(\rho_t,\rho_{0})\leq d_H(\rho_{t-1},\rho_{0})$ in probability (w.r.t $\pi$), which we may substitute into the right-hand side of \eqref{eq:triangle} to obtain
\begin{equation}\label{eq:traingle2}
d_H(\rho_t,\rho_{t-1}) \leq 2~d_H(\rho_{t-1},\rho_{0}) 
\end{equation}
We can now write the following relationship for event for any $\gamma$ as
\begin{equation}\label{eq:subset}
\!\!\!\{d_H(\rho_t,\rho_{t-1}) \!  \!<\! \gamma \} \!\subset \!\{ 2~d_H(\rho_{t-1},\rho_{0}) \!<\! \gamma \}\; .
\end{equation}
Now we compute the prior probability conditioned on  $\ccalS_t$,
\begin{equation}\label{eq:prob}
\mathbb{P}_{\Pi}\{d_H(\rho_t,\rho_{t-1}) \!  \!<\! \gamma \given \ccalS_t \} \leq \mathbb{P}_{\Pi}\Big\{ d_H(\rho_{t-1},\rho_{0}) \!<\! \frac{\gamma}{2} \given \ccalS_t \Big\}\; .
\end{equation}
By shifting the time-indices of the term on the right-hand side of \eqref{eq:prob}, we compute the prior probability conditioned on  $\ccalS_t$
using Lemma \ref{lemma:posterior_consistency} [cf.\eqref{eq:posterior_consistency_uncompressed}] and can be written as 
\begin{equation}\label{eq:triangle_prob}
\lim_{t\rightarrow \infty} \mathbb{P}_{\Pi}\{ d_H(\rho_t,\rho_{0})< \frac{\gamma}{2}  \given \ccalS_t \} =1
\end{equation}
Thus using \eqref{eq:triangle_prob}, we may conclude that the left-hand side of \eqref{eq:prob} then has limit less than or equal to 1, and hence its lim sup satisfies 
\begin{equation}\label{eq:proof_lemma2}
\limsup_{t\rightarrow \infty} \mathbb{P}_{\Pi}(d_H(\rho_t,\rho_{t-1})<\gamma)=1
\end{equation}
However, since GP posterior and the Hellinger metric are continuous, the preceding limit of $\mathbb{P}_{\Pi}(d_H(\rho_t,\rho_{t-1})<\gamma)$ exists, and hence is unique. Therefore, we may conclude the prior probability of the left-hand side of \eqref{eq:subset} also converges to one, yielding Lemma \ref{lemma:lyapunov}.

\end{proof}
\section{Proof of Theorem \ref{theorem:pog_consistency}}\label{sec:apx_thm1}
{\bf \noindent Proof of Theorem \ref{theorem:pog_consistency}\ref{theorem:diminishing}}:
 We relate the per-step behavior $d_H(\rho_{\bbD_t},\rho_{\bbD_{t-1}}) $ of Algorithm \ref{alg:pog} to the sample path of $d_H(\rho_t,\rho_{t-1}) $. First, apply the triangle inequality again to obtain that this distance decomposes into two terms: 
\begin{equation}\label{eq:hellinger_triangle}
d_H(\rho_{\bbD_t},\rho_{\bbD_{t-1}}) \leq d_H(\rho_{\bbD_t},\rho_{\tbD_t}) + d_H(\rho_{\tbD_t},\rho_{\bbD_{t-1}})
\end{equation}
The first term on the right-hand side of \eqref{eq:hellinger_triangle} is exactly the DHMP stopping criterion, and thus is no more than $\eps_t$. Therefore, we have the following containment relationship for events:
\begin{align}\label{eq:subset22}
\!\!\!\{d_H(\rho_{\bbD_t},\rho_{\bbD_{t-1}}) <\alpha \} &\nonumber \\
\qquad \subset \{ d_H(\rho_{\bbD_t},\rho_{\tbD_t})  &+ d_H(\rho_{\tbD_t},\rho_{\bbD_{t-1}}) < \alpha \} \nonumber \\
 \subset \{  d_H(\rho_{\tbD_t},&\rho_{\bbD_{t-1}}) + \epsilon_t < \alpha  \}
\end{align}
Let's compute the prior probability of \eqref{eq:subset22} for any $\alpha>0$ as
\begin{align}\label{eq:pog_evolution1}
\mathbb{P}_{\Pi}\{ & d_H(\rho_{\bbD_t},\rho_{\bbD_{t-1}})  < \alpha \given \ccalS_t \} \nonumber \\
&\leq  \mathbb{P}_{\Pi}\{d_H(\rho_{\bbD_t},\rho_{\tbD_t}) + d_H(\rho_{\tbD_t},\rho_{\bbD_{t-1}})< \alpha \given \ccalS_t\}  \nonumber \\
& \leq  \mathbb{P}_{\Pi}\{d_H(\rho_{\tbD_t},\rho_{\bbD_{t-1}}) +  \eps_t< \alpha  \given \ccalS_t\} 
\end{align}
where in the last expression we have applied the DHMP stopping criterion. Now, subtract the constant $\eps_t$ from both sides inside the event on the right-hand side of \eqref{eq:pog_evolution1} to define  the event on the right-hand side of \eqref{eq:pog_evolution1} as $\tilde{\eta}_t:=\{ d_H(\rho_{\tbD_t},\rho_{\bbD_{t-1}}) <\alpha -\epsilon_t \given \ccalS_t \}$. Subsequently, define $\gamma=\alpha-\epsilon_t>0$.

The event sequences $\eta_t$ (defined in Lemma \ref{lemma:lyapunov}) and $\tilde{\eta}_t$ quantify the effect of the \emph{same} Bayesian {a posteriori} updates across time, but using posterior distributions parameterized by different kernel dictionaries, namely, $\ccalS_{t-1}\cup \bbx_{t} $ [cf. \eqref{eq:GP_posterior_t}] versus $\tbD_t=[\bbD_{t-1} ; \bbx_t]$ [cf. \eqref{eq:GP_posterior_D}]. 
To the event $\tilde{\eta}_t$ we can apply Assumption \ref{as:unbiasedness} to write
%
%
\begin{align}\label{eq:pog_evolution}
 \mathbb{P}_{\Pi}\{d_H(\rho_{\tbD_t},\rho_{\bbD_{t-1}}) < \alpha - \eps_t \given \ccalS_t\} &=  \mathbb{P}_{\Pi}\{\tilde{\eta}_t \given \ccalS_t\} \nonumber \\
 &\leq  \mathbb{P}_{\Pi}\{\eta_t \given \ccalS_t\}
\end{align}
Now, suppose $\eps_t \rightarrow 0$ as $t \rightarrow \infty$ so that $\gamma \rightarrow \alpha$. Thus by using Lemma \ref{lemma:lyapunov} on the right-hand side of \eqref{eq:pog_evolution} we conclude that
\begin{equation}
\limsup_{t\rightarrow \infty} \mathbb{P}_{\Pi}(d_H(\rho_{\bbD_t},\rho_{\bbD_{t-1}})<\alpha - \eps \given \ccalS_t)=1
\end{equation}

Again, we use continuity of the GP posterior and the Hellinger metric to conclude the preceding limit exists, and therefore
\begin{equation}
\lim_{t\rightarrow \infty} \mathbb{P}_{\Pi}(d_H(\rho_{\bbD_t},\rho_{\bbD_{t-1}})<\alpha - \eps \given \ccalS_t)=1
\end{equation}
Therefore, for choice of compression budget $\eps_t \rightarrow 0$, we have \\$\mathbb{P}_{\Pi}(d_H(\rho_{\bbD_t},\rho_{\bbD_{t-1}})<\gamma\given \ccalS_t)\rightarrow 1$ for any 
 $\gamma> 0$. Substitute
in the definition of 
$\gamma=\alpha-\eps_t$ to obtain Theorem \ref{theorem:pog_consistency}\ref{theorem:diminishing}.
  $\hfill\blacksquare$

{\bf \noindent Proof of Theorem \ref{theorem:pog_consistency}\ref{theorem:constant}}:
We again relate the asymptotic probabilistic behavior of $d_H(\rho_{\bbD_t},\rho_{\bbD_{t-1}})$ defined by Algorithm \ref{alg:pog} to the true uncompressed  sequence $d_H(\rho_t,\rho_{t-1})$ where $\rho_t$ is given in \eqref{eq:GP_posterior_t}. Begin with the expression \eqref{eq:pog_evolution1}, followed by subtracting $\eps$ from both sides to obtain:
\begin{align}\label{eq:pog_evolution_constant}
\mathbb{P}_{\Pi}\{ & d_H(\rho_{\bbD_t},\rho_{\bbD_{t-1}})  < \alpha \given \ccalS_t \} \nonumber \\
& \leq  \mathbb{P}_{\Pi}\{d_H(\rho_{\tbD_t},\rho_{\bbD_{t-1}}) < \alpha - \eps \given \ccalS_t\} 
\end{align}
The right-hand side of \eqref{eq:pog_evolution_constant} is $ \mathbb{P}_{\Pi} \{ \tilde{\eta}_t   \}$, to which we may apply Assumption \ref{as:unbiasedness} which states that $\mathbb{P}_{\Pi} \{\eta_t  \} \geq \mathbb{P}_{\Pi} \{ \tilde{\eta}_t   \}$, provided $\gamma =\alpha - \eps > 0$, to obtain
\begin{align}\label{eq:pog_evolution_constant2}
\mathbb{P}_{\Pi}\{&d_H(\rho_{\tbD_t},\rho_{\bbD_{t-1}}) < \alpha - \eps \given \ccalS_t\}   \\
&  
\quad \leq \mathbb{P}_{\Pi}\{d_H(\rho_t,\rho_{t-1})<\alpha - \eps \given \ccalS_t \}\nonumber
\end{align}
By Lemma 2, the supremum of the probability of the right-hand side of \eqref{eq:pog_evolution_constant2} approaches $1$ as $t\rightarrow \infty$ for $\gamma=\alpha-\eps>0$.
Thus now the left-hand side of \eqref{eq:pog_evolution_constant} can be written as:
\begin{align}\label{eq:pog_evolution_constant3}
\limsup_{t\rightarrow \infty} \mathbb{P}_{\Pi}\{ & d_H(\rho_{\bbD_t},\rho_{\bbD_{t-1}})  < \alpha \given \ccalS_t \} =1
\end{align}
Again, we exploit the continuity of the GP posterior and the Hellinger metric to conclude the preceding limit
exists. Theorem \ref{theorem:pog_consistency}\ref{theorem:constant} follows from substituting in $\alpha=\gamma+\eps$ into \eqref{eq:pog_evolution_constant3}. $\hfill\blacksquare$

\section{Proof of Theorem \ref{theorem:pog_parsimony}}\label{sec:apx_thm2}
This proof is inspired by, but conceptually distinct from, that of Theorem 3 in \citep{koppel2017parsimonious}. We consider two subsequent iterates generated by Algorithm \ref{alg:pog} and connect the model order growth from one to the next to the approximation error in the Hellinger metric with the newest point $\bbx_{t}$ removed. We then analyze conditions under which this condition is false for all subsequent times, and are able to connect this condition to the distance between the Hilbert subspace defined by the current kernel dictionary $\bbD_t$ and the kernel evaluation of the newest point $\kappa(\bbx_t, \cdot)$ through the definition of the Hellinger metric for multivariate Gaussians. Then, we may apply a covering number argument, exploiting the fact that the feature space $\ccalX$ is compact, to establish that the stopping criterion of DHMP is always violated, and hence no additional points are added to the kernel dictionary after a certain time, in a manner similar to \citep{1315946}[Theorem 3.1].

 Begin by considering two subsequent posterior distributions $\rho_{\bbD_t}$ and $\rho_{\bbD_{t+1}}$ generated by Algorithm \ref{alg:pog} of model order $M_t$ and $M_{t+1}$, respectively, assuming a constant compression budget $\eps>0$. Suppose that the model order of $\rho_{\bbD_{t+1}}$ is not larger than the previous posterior $\rho_{\bbD_t}$, i.e., $M_{t+1} \leq M_t$. This relation is valid when the stopping criterion of Algorithm \ref{alg:komp} is \emph{violated} for the kernel dictionary with the newest sample added $\tbD_{t+1} = [\bbD_t ; \bbx_t]$ of size $M_t+ 1$. The negation of the stopping criterion for Algorithm \ref{alg:komp} is stated as
 \begin{equation}\label{eq:compression_error}
\hspace{3cm} \min_{j=1,\dots, M_t +1 } \gamma_j \leq \eps \; .
  \end{equation}
Note that \eqref{eq:compression_error} lower bounds the approximation error $\gamma_{M_t + 1}$ associated with removing $\bbx_{t}$. Therefore, if $\gamma_{M_t + 1} \leq \eps$, then \eqref{eq:compression_error} holds, and the model order does not grow. Thus, we may consider $\gamma_{M_t + 1}$ in lieu of $\gamma_j$ for each $j$.

The definition of $\gamma_{M_t + 1} $, using its expression in Algorithm \ref{alg:komp} with dictionary $\tbD_{t+1}$, is $d_H(\rho_{\tbD_{-M_t+1}}, \rho_{\tbD_{t+1}})$. Using the expression for the Hellinger metric for multivariate Gaussians in \eqref{eq:hellinger_gaussian}, we may glean that  $\gamma_{M_t + 1} $ depends only on the difference between the mean and covariances with and without sample point $\bbx_t$. That is, $\gamma_{M_t + 1} \propto (\mu_{t+1\given \bbD_t} - \mu_{\bbD_t}, \Sigma_{t+1\given \bbD_t} - \Sigma_{\bbD_t} ) $.

Unfortunately, there is no closed form expression that relates these differences to the distance between the Hilbert subspace defined by the current dictionary $\ccalH_{\bbD_t}:=\text{span}\{\kappa(\bbd_j, \cdot)\}_{j=1}$ and the kernel evaluation of the latest point $\kappa(\bbx_t, \cdot)$. However, the distances between means and covariances are completely determined by this distance,
 $$\text{dist}(\kappa(\bbx_t, \cdot),\ccalH_{\bbD_t}):=\min_{\bbv^\in\reals^{M_t}}\|\kappa(\bbx_t,\cdot) -  \bbv^T\mathbf{\kappa}_{\bbD_t}(\cdot)\|_{\ccalH}.$$
  Therefore, if there exists some $\eps'>0$ such that $\text{dist}(\kappa(\bbx_t, \cdot),\ccalH_{\bbD_t}) \leq \eps'$, then for some other $\eps>0$,  $\gamma_{M_t + 1} \leq \eps$. 

The contrapositive of this statement is that $\gamma_{M_t + 1} > \eps$, implies $\text{dist}(\kappa(\bbx_t, \cdot),\ccalH_{\bbD_t}) > \eps'$ for some $\eps'>0$. Thus, the DHMP stopping criterion is violated whenever for distinct $\bbd_j$ and $\bbd_k$ for $j,k\in\{1,\dots,M_t\}$ satisfy $\|\kappa(\bbd_j, \cdot) - \kappa(\bbd_k,\cdot)\|_{\ccalH} > \eps'$ for some $\eps'>0$, and for such cases, the model order grows from times $t$ to $t+1$.
We may now proceed as in the proof of \citep{1315946}[Theorem 3.1]: since $\ccalX$ is compact and $\kappa$ is continuous, the range of the kernel transformation  $\phi(\ccalX):=\kappa(\ccalX,\cdot)$ of the feature space $\ccalX$ is compact. Therefore, the number of balls of radius $\eps'$ needed to cover $\phi(\ccalX)$ is finite (see \citep{anthony2009neural}, for instance), and depends on the covering number of $\phi(\ccalX)$ at scale $\eps'$. Therefore, for some $M^\infty$, if $M_t=M^{\infty}$, the condition $\text{dist}(\kappa(\bbx_t, \cdot),\ccalH_{\bbD_t}) \leq \eps'$ is valid, in which case $\gamma_{M_t + 1} \leq \eps$ holds, and thus \eqref{eq:compression_error} is true. Therefore, $M_t \leq M^{\infty}$ for all $t$. $\hfill\blacksquare$

We sharpen this dependence by noting that \citep{1315946}[Proposition 2.2] states that for a Lipschitz continuous Mercer kernel $\kappa$ on a compact set $\ccalX\subset\reals^p$, there exist a constant $Y$ such that for any training set $\{\bbx_u\}_{u\leq t}$ and any $\nu>0$, the number of elements in the dictionary $M$ satisfies
$$M\leq Y\left(\frac{1}{\nu}\right)^p$$
where $Y$ is a constant depending on the radius of $\ccalX$ and the kernel hyper-parameters. By the previous reasoning, we have $\nu=\epsilon'$. Taken with the fact that the Euclidean and Hellinger distances are constant factors apart, we have  know that $\epsilon'=\ccalO(\epsilon)$, which overall allows us to conclude
$$M\leq \ccalO\left(\frac{1}{\epsilon}\right)^p$$

\bibliographystyle{spbasic} 
\bibliography{Manuscript}




\end{document}